\documentclass[11pt,letterpaper]{article}

\usepackage{deepthink}

\usepackage[utf8]{inputenc}
\usepackage[T1]{fontenc}
\usepackage[numbers, sort&compress]{natbib}
\usepackage{url}
\usepackage{booktabs}
\usepackage{amsmath}
\usepackage{amsfonts}
\usepackage{nicefrac}
\usepackage{bm}
\usepackage{wrapfig}
\usepackage{listings}
\usepackage{caption}
\usepackage{subcaption}
\usepackage{multirow}
\usepackage{algorithm}
\usepackage{algpseudocode}
\usepackage[normalem]{ulem}

\usepackage{amssymb, amsthm}
\usepackage{cancel}
\usepackage{tikz}
\usepackage[nameinlink]{cleveref}   
\tcbuselibrary{breakable}

\graphicspath{{figs/}}

\usepackage{amsmath,amssymb,amsthm}

\theoremstyle{definition}

\newcommand{\R}{\mathbb{R}}

\newcommand{\E}{\mathbb{E}}

\newcommand{\norm}[1]{\left\lVert#1\right\rVert}

\definecolor{commentgreen}{RGB}{0, 128, 128}
\definecolor{codeblack}{RGB}{0, 0, 0}
\definecolor{brightviolet}{RGB}{170, 0, 255}      
\definecolor{violetcomment}{RGB}{200, 80, 255}    

\makeatletter
\newcommand{\algrule}[1][.4pt]{\par\vskip.5\baselineskip\hrule height #1\par\vskip.5\baselineskip}
\makeatother

\newtheorem{condition}{Condition}
\newtheorem{proposition}{Proposition}

\title{Data-Forcing Distillation: Restoring Diversity and Fidelity in Few-Step Video Generation}

\authorblock{
  \href{https://csy2077.github.io/siyichen.github.io/}{\textbf{Siyi Chen}}\textsuperscript{1},
  \href{https://stevenlsw.github.io}{\textbf{Shaowei Liu}}\textsuperscript{2,3},
  \href{https://umjiayx.github.io}{\textbf{Yixuan Jia}}\textsuperscript{1},
  \href{https://www.cs.toronto.edu/~zianwang/}{\textbf{Zian Wang}}\textsuperscript{2},
  \href{https://www.cs.toronto.edu/~linghuan/}{\textbf{Huan Ling}}\textsuperscript{2},
  \href{https://qingqu.engin.umich.edu}{\textbf{Qing Qu}}\textsuperscript{1},
  \href{http://www.cs.toronto.edu/~jungao/}{\textbf{Jun Gao}}\textsuperscript{1,2}
}

\affiliation{
  \textsuperscript{1}University of Michigan \quad $\cdot$ \quad
  \textsuperscript{2}NVIDIA \quad $\cdot$ \quad
  \textsuperscript{3}University of Illinois Urbana-Champaign
}

\abstracttext{
\noindent Recent progress has shown promise in distilling multi-step video diffusion models into efficient few-step students.
Among them, Distribution Matching Distillation (DMD) and its successor DMD2 achieved strong generation quality and fast convergence.
However, due to the nature of the reverse Kullback--Leibler (KL) objective, these methods exhibit two persistent failure modes: a substantial drop in sample diversity, and visibly over-saturated outputs that deviate from real-video appearance.
In this work, we propose \textit{Data-Forcing Distillation} (DFD), a simple post-training framework that restores diversity and fidelity in DMD \textit{with only a single-line of code change}.
At its core is the \emph{teacher score discrepancy} to guide the student toward the real-data distribution, pulling it to missing modes (mitigating mode collapse) and away from problematic modes absent in real data (avoiding over-saturation).
We provide an in-depth theoretical analysis of our framework and validate our approach on both text-to-video and image-to-video generation.
With \emph{only 100--300 steps} of finetuning, DFD effectively restores diversity and fidelity on both Wan2.1-1.3B and Cosmos-Predict2.5-2B model, resolving the over-saturation artifacts with significantly better video dynamics and appearance, and even outperforms the teacher model. We refer the readers to the supplementary material for full video comparisons.
}


\date{\today}
\correspondence{\href{mailto:siyiche@umich.edu}{siyiche@umich.edu}}
\resources{\href{https://csy2077.github.io/dfd-project-page/}{Project Page}}

\DeepthinkSetHeaderFooterSep{0.15em}{0.20em}

\headerlogo{Deepthink_landscape_do_not_delete_compressed.png}{https://deepthink-umich.github.io}

\begin{document}

\makeDeepthinkHeader
\enlargethispage{0.9in}
\vspace{-0.18in}

\begingroup
\centering
\setlength{\abovecaptionskip}{4pt}
\includegraphics[width=0.96\textwidth]{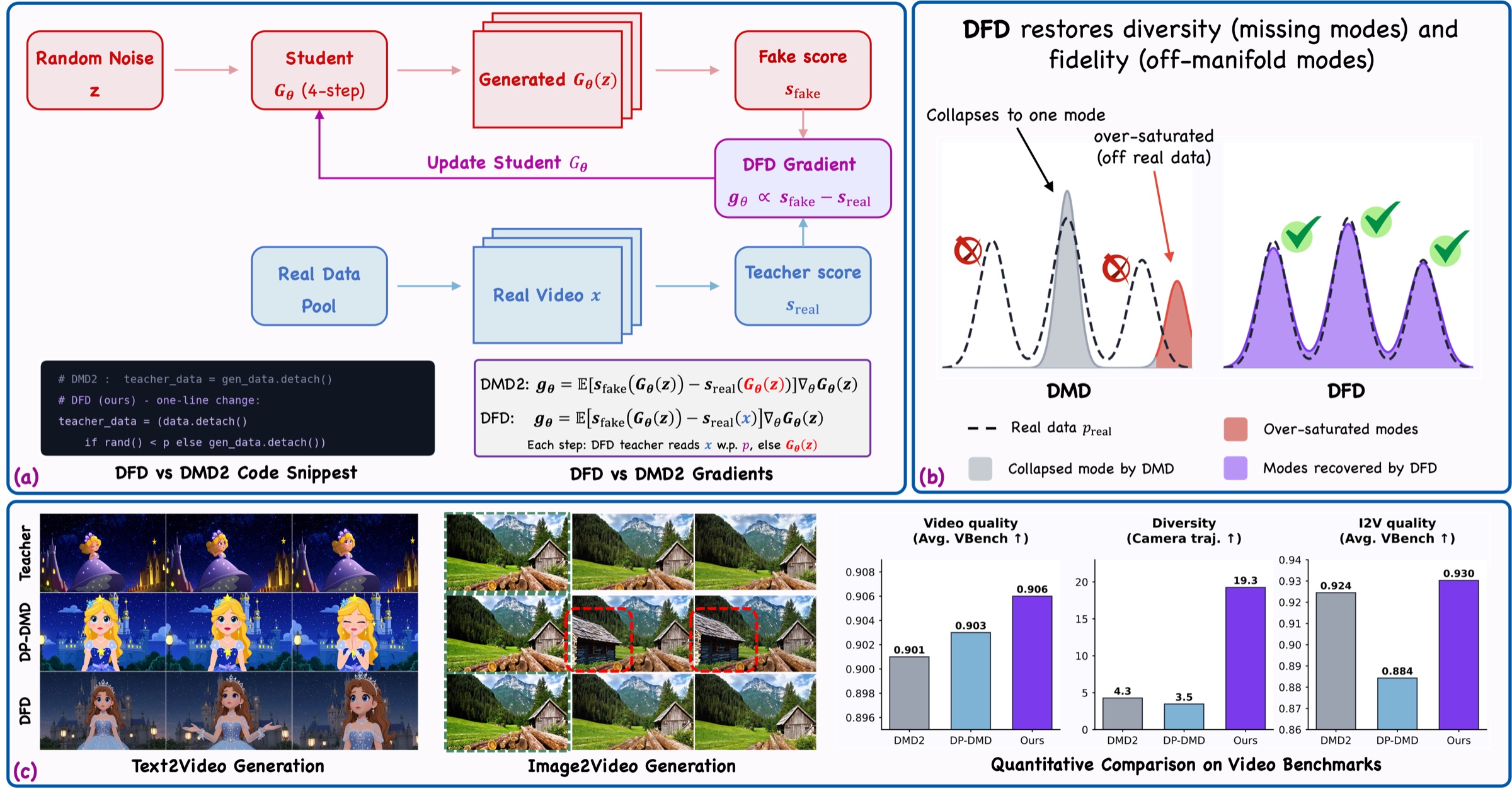}
\captionof{figure}{\textbf{Overview of DFD.} (a) DFD is a one-line change to DMD2, where the teacher scores a real video $x$ (w.p. $p$) instead of the student's generation $G_\theta(z)$, giving the update $g_\theta\propto s_\text{fake}-s_\text{real}$. (b) This counters reverse-KL mode-seeking, recovering collapsed (diversity) and over-saturated off-manifold (fidelity) modes. (c) DFD improves quality and substantially increases diversity on text-to-video and image-to-video benchmarks.}
\label{fig:teaser}
\par
\endgroup
\vspace{0.05in}

\tableofcontents
\newpage




\section{Introduction}
\label{sec:introduction}


Recent progress in large-scale diffusion and flow-based models has substantially advanced video generation~\cite{song2020score, lipman2022flow, rombach2022high, wan2025wan, cosmos-predict2p5, kong2024hunyuanvideo, yang2024cogvideox, brooks2024video}. Modern video generators can synthesize photorealistic, high-fidelity, and diverse videos from a text prompt or a single image.
Yet, this capability comes at a heavy computational cost: high-quality generation typically requires multiple denoising steps, resulting in high inference latency and limited deployment in interactive, real-time, or large-scale production settings.
Accelerating diffusion sampling while preserving the quality and diversity of the original multi-step model has thus become a central research challenge.

Two families of distillation methods have emerged to address this challenge. \emph{Trajectory-based} methods train the few-step student to regress the sampling trajectory of a multi-step teacher~\cite{salimans2022progressive,meng2023distillation,song2023consistency,song2023improved,lu2024simplifying,geng2025mean,sabour2025align}. While effective at reducing sampling steps, these methods require the student to approximate a complex multi-step trajectory with only a few function evaluations, which becomes challenging for high-dimensional video generation and can result in noticeable quality degradation.
\emph{Distribution-based} methods~\cite{yin2024one,yin2024improved,luo2023diff,wang2025uni} relax this trajectory constraint and train the student so that its output distribution matches the teacher's distribution. DMD~\cite{yin2024one} and its successor DMD2~\cite{yin2024improved} are representative examples, achieving strong few-step quality and fast convergence; we therefore build on this family in our work.
However, when applied to large-scale video generation, distribution-based methods still exhibit two critical failure modes: a substantial drop in sample diversity compared to the teacher model, and visibly over-saturated outputs that deviate from real-video appearance.


The root cause of both failure modes is the reverse Kullback-Leibler (KL) objective used in DMD, which is well known to be \emph{mode-seeking}~\cite{poole2022dreamfusion,wang2023prolificdreamer}: the student seeks the highest-density modes of $p_{\mathrm{real}}$, as the reverse KL takes its expectation over student-generated samples (Sec.~\ref{sec:dmd_formulation}) and primarily optimizes only the modes that the student covers and cannot penalize missing modes in the real distribution. 
To fix the problem, our intuition is simple: directly incorporating real video samples into the training objective during distillation, so that the student can be explicitly pulled toward the diverse, high-quality real data distribution rather than supervising itself using its own generated data.

Building on this motivation, we propose \textit{Data-Forcing Distillation} (DFD), a video distillation framework that restores both diversity and fidelity with a single line of code change on top of DMD. The core of DFD is the \emph{teacher score discrepancy} that computes the score difference between a real video and the student's generated video, and we directly incorporate it in the distillation objective.
Intuitively, this score discrepancy points the student to the real-data distribution, pulling it toward missing modes (mitigating mode collapse) and away from problematic modes absent in real data (avoiding over-saturation). 
In practice, DFD is implemented with a single-line code change: feeding a real video sample to the teacher in place of the student's own generation when computing the distribution matching objective.
While the idea of leveraging real data was already explored in DMD2~\cite{yin2024improved}, it was implemented through an auxiliary GAN~\cite{goodfellow2014generative}, which destabilizes training and only \emph{implicitly} transfers real-data information to the generator through a discriminator. In contrast, DFD \emph{explicitly} provides real-data guidance to the distribution matching objective, preserving the efficiency and simplicity of DMD while mitigating both failure modes.

We conduct extensive experiments on two different video generation tasks with different pretrained video models to demonstrate the effectiveness of DFD. 
On text-to-video generation, we distill the pretrained Wan2.1-1.3B~\cite{wan2025wan} model into a 4-step student model.
On image-to-video generation, we apply the same recipe and distill the pretrained Cosmos-Predict2.5-2B~\cite{cosmos-predict2p5} model. 
With only 50-100 steps of post-training on top of DMD2, our DFD consistently and significantly improves visual quality, restores diversity across scene composition and camera motion quality of DMD2, and substantially alleviates over-saturation artifacts. 
We will release our code and model checkpoints to support future research on efficient, high-fidelity, and diverse video generation.

\section{Related Work}
\label{sec:relatedworkd}
\subsection{Trajectory-Based Methods}
Trajectory-based distillation methods aim to accelerate the sampling process by training a few-step student model to match the denoising trajectory of a multi-step teacher model \cite{song2023consistency, luo2023latent, meng2023distillation, li2024t2v}. Among these, Consistency Models (CMs) \cite{song2023consistency, song2023improved, geng2024consistency, geng2025mean, sabour2025align, kim2023consistency} learns a consistency function $f_{\bm \theta}(\bm x_t, t) \to \bm x_0$ 
that directly maps any point $\bm x_t$ on the teacher's probability flow ODE trajectory to the initial data point $\bm x_0$.
Initially, discrete-time CMs were trained by minimizing the discrepancy between outputs at adjacent anchor points $t$ and $t-\Delta t$ along the teacher trajectory. 
Recently, continuous-time CMs
\cite{lu2024simplifying,kim2023consistency} offer a clean upgrade by taking the limit as $\Delta t \to 0$, simplifying the objective to enforce instantaneous self-consistency. 
Building upon these foundations,  a notable work is the rCM \cite{zheng2025large}, which is the first work to scale up continuous-time consistency models to large image and video diffusion models. Despite their significant empirical success and clean mathematical formulations, trajectory-based distillation methods still face open challenges, often exhibiting noticeable performance degradation when scaled up to large pretrained image synthesis or video generation models \cite{ wang2024phased}.
\subsection{Distribution-Matching Methods}
Distribution-matching distillation compresses a multi-step teacher into a few-step student by aligning the student's output distribution with the teacher's. DMD~\cite{yin2024one} minimizes a reverse KL between the two and achieves stable training and strong sample quality; DMD2~\cite{yin2024improved} augments it with an auxiliary GAN discriminator to mitigate the mode-seeking behavior of reverse KL, but inherits the well-known instability of adversarial training and only injects the real-data signal \emph{implicitly} through the discriminator. These issues are amplified for video diffusion, where the high spatiotemporal dimensionality makes mode collapse and over-saturation considerably more severe.
 
A parallel line of work generalizes the divergence itself. $f$-distill~\cite{xu2025one} casts DMD as a special case of integral $f$-divergence and reports gains from alternatives such as the Jeffreys divergence, while Uni-Instruct~\cite{wang2025uni} unifies Score Implicit Matching~\cite{luo2024one} and Score Identity Distillation~\cite{zhou2024score} under a single objective. Diversity-Preserving DMD~\cite{wu2026diversity} instead couples DMD with trajectory-based distillation to recover modal diversity, though scaling such hybrids to video remains open. Most recently, Transition Matching Distillation (TMD)~\cite{Nie2026TransitionMD} extends the paradigm temporally by matching conditional transition probabilities across sub-intervals of the denoising trajectory, enabling distillation of large-scale video models -- yet it still rests on the same reverse-KL backbone and inherits its diversity limitations.

\section{Background: Distribution Matching Distillation}
\label{sec:dmd_formulation}
We first review Distribution Matching Distillation (DMD), the reverse-KL objective it optimizes, and the mode-collapse issue that motivates our approach.

\paragraph{Distribution Matching Distillation (DMD)~\cite{yin2024one,yin2024improved}.}
To distill a pretrained multi-step teacher into a few-step student $G_\theta$, DMD matches the student's output distribution $p_{\mathrm{fake}}$ to the teacher's data distribution $p_{\mathrm{real}}$ via a KL objective. Since neither density is tractable, DMD reformulates the gradient in terms of \emph{score functions}, which can be approximated by diffusion denoisers. Concretely, for a noise-perturbed sample $\bm{x}_t = \alpha_t\, G_\theta(\bm{z}, c) + \sigma_t\, \bm{\epsilon}$ with $\bm{\epsilon} \sim \mathcal{N}(\bm{0}, \bm{I})$, the student gradient takes the form
\begin{equation} \label{eq:dmd_grad}
    \nabla_\theta \mathcal{L}_{\mathrm{DMD}}
    \;=\; \E_{t, \bm{z}, \bm{\epsilon}}\!\left[\, w(t)\,\big(\nabla_{\bm{x}} \log p_{\mathrm{fake}}(\bm{x}_t) - \nabla_{\bm{x}} \log p_{\mathrm{real}}(\bm{x}_t)\big)\, \nabla_\theta G_\theta(\bm{z}, c) \right],
\end{equation}
where $\nabla_{\bm{x}} \log p_{\mathrm{real}}$ is learned by the frozen teacher denoiser, $\nabla_{\bm{x}} \log p_{\mathrm{fake}}$ is learned by an auxiliary denoiser trained \emph{online} on student samples via denoising score matching, and $w(t)$ is a noise-level weighting. Training alternates between updating $\theta$ with Eq.~\eqref{eq:dmd_grad} and updating the auxiliary denoiser to track the moving $p_{\mathrm{fake}}$.

\paragraph{Reverse Kullback-Leibler (KL) Loss.}
DMD instantiates Eq.~\eqref{eq:dmd_grad} as the gradient of the \emph{reverse} KL,
\begin{equation} \label{eq:reverse_kl}
    D_{\mathrm{KL}}\!\big(p_{\mathrm{fake}} \,\|\, p_{\mathrm{real}}\big)
    \;=\; \E_{\substack{\bm{x} = G_\theta(\bm{z}, c) \\ \bm{z} \sim \mathcal{N}(\bm{0}, \bm{I})}}\!\left[\, \log \frac{p_{\mathrm{fake}}(\bm{x})}{p_{\mathrm{real}}(\bm{x})} \right],
\end{equation}
along the diffusion path whose differentiation yields
\begin{equation} \label{eq:reverse_kl_grad}
    \nabla_\theta D_{\mathrm{KL}}(p_{\mathrm{fake}} \,\|\, p_{\mathrm{real}})
    \;=\; \E_{\bm{z}}\!\left[\big(\nabla_{\bm{x}} \log p_{\mathrm{fake}}(\bm{x}) - \nabla_{\bm{x}} \log p_{\mathrm{real}}(\bm{x})\big)\, \nabla_\theta G_\theta(\bm{z}, c) \right],
    \quad \bm{x} = G_\theta(\bm{z}, c).
\end{equation}
This formulation gives strong few-step quality and fast convergence, but the reverse KL is well known to be \emph{mode-seeking}: the student collapses onto the highest-density modes of $p_{\mathrm{real}}$, producing low-diversity, often over-saturated samples~\cite{he2024training, lu2025adversarial}. The core reason for this mode-seeking behavior is that reverse KL divergence is evaluated as an expectation over the student's generated distribution. Consequently, if the student drops a mode, the loss function incurs no penalty in that region, yielding zero gradient signal to recover the missing mode. The core reason for mode-seeking is that reverse KL samples from the data generated by the student, and when there is a missing mode in the student, the KL divergence can not pull it back.

\paragraph{Regularization with Real Data}
To overcome the stated mode collapse, over-saturated problem introduced by reverse KL loss, DMD2~\cite{yin2024improved} partially compensates with an auxiliary GAN \cite{goodfellow2014generative} discriminator, but adversarial training is unstable and only \emph{implicitly} injects the real-data signal through the discriminator rather than through the distribution-matching gradient itself.

\section{Data-Forcing Distillation}
\label{sec:method}
Our core intuition is to incorporate real data directly into the reverse KL gradient through a differentiable regularizer, where the student's distribution can be explicitly pulled toward the diverse, high-quality real data distribution rather than supervising itself in the teacher.  In this section, we first formally formulate our proposed data-forcing distillation, then give a practical implementation. 

\subsection{The Data-Forcing Distillation (DFD) Framework}
Our DFD injects a real-data regularization term into the original DMD gradient.
\begin{equation}\label{eq:dfd_grad_regular}
\begin{aligned}
g_{\mathrm{DFD}}(\theta) \;=\;\;
& \underbrace{\E_{\bm{z} \sim \mathcal{N}(\bm{0}, \bm{I})}\!\Bigl[\bigl(\nabla_{\bm{x}} \log p_{\mathrm{fake}}(G_\theta(\bm{z},c)) - \nabla_{\bm{x}} \log p_{\mathrm{real}}(G_\theta(\bm{z},c))\bigr)\, \nabla_\theta G_\theta(\bm{z}, c)\Bigr]}_{\text{native DMD gradient}\;g_{\mathrm{DMD}}(\theta)} \\
- \;\;
& \underbrace{\E_{\substack{\bm{x} \sim p_{\mathrm{real}}(\cdot \mid c), \, \bm{z} \sim \mathcal{N}(\bm{0}, \bm{I})}}\!\Bigl[\Delta_{p_{\mathrm{real}},\, \bm{x},\, G_\theta(\bm{z})}\, \nabla_\theta G_\theta(\bm{z}, c)\Bigr]}_{\text{real-data regularizer}}.
\end{aligned}
\end{equation}
\begin{equation}\label{eq: regularizes term}
\Delta_{p_{\mathrm{real}},\bm{x},G_\theta(\bm{z})} \;=\; \nabla_{\bm{x}} \log p_{\mathrm{real}}(\bm{x})-\nabla_{\bm{x}} \log p_{\mathrm{real}}(G_\theta(\bm{z},c)), \quad  \bm{x} \sim p_{\mathrm{real}}(\cdot \mid c).
\end{equation}
We refer to $\Delta_{p_{\mathrm{real}},\bm{x},G_\theta(\bm{z})}$ as the \textbf{teacher score discrepancy}: it measures the gap, in the teacher's score field, between a real sample and the student's generation. When the student matches the real distribution ($p_{\mathrm{fake}}=p_{\mathrm{real}}$), this discrepancy is minimized to zero in expectation, $\E_{\bm{x}, \bm{z}}[ \Delta_{p_{\mathrm{real}},\bm{x},G_\theta(\bm{z})} ] = 0$ (proof in the Appendix~\ref{app: additional theory}). Empirically, drawing $\bm{x}$ from diverse, high-quality real-world data and minimizing $ \Delta_{p_{\mathrm{real}},\bm{x},G_\theta(\bm{z})}$ pulls the student toward modes it has missed in the real data (mitigating mode collapse) and away from problematic modes such as over-saturated outputs that are not in the real data. More importantly, properties of real data, such as temporal coherence, physical plausibility, photorealism, can be explicitly distilled into the student through this term.


\paragraph{Discussion: No Need for a GAN.}
DMD2~\cite{yin2024improved} provides \emph{implicit} real-data supervision through an auxiliary GAN loss; in contrast, DFD provides \emph{explicit} real-data supervision through the score-discrepancy term (Eq.~\ref{eq: regularizes term}). As a result, we drop the GAN loss completely from our distillation pipeline and it will not affect our performance. We validate this in the ablation study in \cref{sub: ablate study}.

\subsection{The Underlying Assumption of DFD}
\label{subsub: Validity Condition}
The teacher score discrepancy $\Delta_{p_{\mathrm{real}},\bm{x},G_\theta(\bm{z})}$ is well-behaved \emph{in expectation}. However, what governs whether this regularizer helps or hurts training is not its expectation but its \emph{variance}---a zero-mean signal can still produce a high-noisy gradient that destabilizes optimization. We therefore need  $\E_{\bm{x}, \bm{z}}\left[||\Delta_{p_{\mathrm{real}},\bm{x},G_\theta(\bm{z})}||^2 \right]$ to stay small. To see what controls this variance, a first-order expansion of the teacher's score around $G_\theta(\bm{z},c)$ shows that, to leading order, the per-sample discrepancy is approximately proportional to the gap between paired real and generated samples,
\begin{equation} \label{eq:term_expansion}
\bigl\|\Delta_{p_{\mathrm{real}},\bm{x},G_\theta(\bm{z})}\bigr\|
    = \bigl\|\nabla_{\bm{x}} \log p_{\mathrm{real}}(\bm{x})-\nabla_{\bm{x}} \log p_{\mathrm{real}}(G_\theta(\bm{z},c))\bigr\|
    \;\approx\; \alpha_t\,L(c,t)\,\|\bm{x} - G_\theta(\bm{z},c)\|
\end{equation}
where $\bm{x} \sim p_{\mathrm{real}}(\cdot \mid c) $, $\alpha_t$ is the diffusion schedule coefficient, and $L(c, t)$ is the Lipschitz constant of the teacher's score network in practice. Consequently the variance of the regularizer is approximately proportional to the squared gap $\E_{\bm{x}, \bm{z}}[\|\bm{x} - G_\theta(\bm{z},c)\|^{2}]$, and controlling $\E_{\bm{x}, \bm{z}}\left[|| \Delta_{p_{real},\bm{x},G_{\theta}(\bm{z})}||^2 \right]$ reduces to controlling this gap. We formalize this requirement as the \emph{validity condition}:
\begin{equation}
    \E_{\bm{x}, \bm{z}}\!\left[\norm{\bm{x} - G_\theta(\bm{z},c)}^{2} \;\big|\; c\right] \;\leq\; \delta(c)^{2},
\label{eq:validity condition}
\end{equation}
where $\delta(c)$ is a conditioning-dependent bound. Combining the approximation above with Eq.~\ref{eq:term_expansion} yields the corresponding (approximate) control on the regularizer's variance:
$\E_{\bm{x}, \bm{z}}\!\bigl[\bigl\|\Delta_{p_{\mathrm{real}},\bm{x},G_\theta(\bm{z})}\bigr\|_{2}^{2} \,\big|\, c\bigr]
    \;\lesssim\; \alpha_t^{2}\,L(c,t)^{2}\,\delta(c)^{2}$.
The complete derivation is in Appendix.~\ref{app: additional theory}. Since we evaluate the score along the full diffusion path rather than only at its endpoint with pure noise, neither $\alpha_t$ nor $L(c,t)$ will explode.

\subsection{Practical Implementation}

Eq.~\eqref{eq:dfd_grad_regular} admits a clean simplification. Because the regularizer $\Delta_{p_{\mathrm{real}},\, \bm{x},\, G_\theta(\bm{z})}$ is evaluated at the same noise level and condition $c$ as the DMD gradient, it exactly cancels the score $\nabla_{\bm{x}} \log p_{\mathrm{real}}(G_\theta(\bm{z}, c))$ . The DFD gradient therefore reduces to
\begin{equation}\label{eq:dfd_grad_simple}
    g_{\mathrm{DFD}}(\theta)
    \;=\;
    \E_{\substack{\bm{x} \sim p_{\mathrm{real}}(\cdot \mid c) \\  \bm{z} \sim \mathcal{N}(\bm{0}, \bm{I})}}\!
    \left[\bigl(\nabla_{\bm{x}} \log p_{\mathrm{fake}}(G_\theta(\bm{z}, c)) - \nabla_{\bm{x}} \log p_{\mathrm{real}}(\bm{x})\bigr)\, \nabla_\theta G_\theta(\bm{z}, c)\right].
\end{equation}


\begin{wrapfigure}{r}{0.6\columnwidth}
    \centering
    \includegraphics[width=0.58\columnwidth]{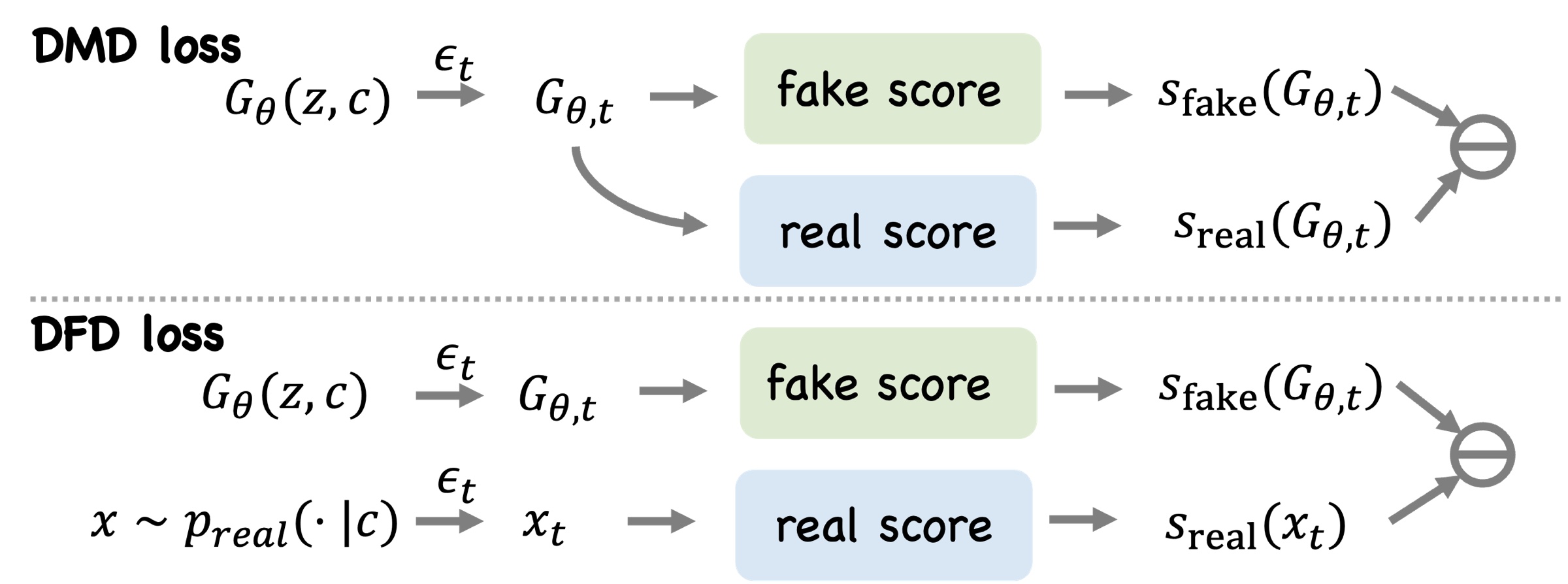}
    \caption{\textbf{The comparison between our DFD and the original DMD}. Our DFD computes the real score directly using the videos sampled from the real data distribution, while the original DMD computes the real score using the generated videos from the student.}
    \label{fig: diff between dmd/dfd}
\end{wrapfigure}


The only difference compared to the DMD gradient (Eq.~\ref{eq:reverse_kl_grad}) is the second score: instead of evaluating $\nabla_{\bm{x}} \log p_{\mathrm{real}}$ at the student's own output $G_\theta(\bm{z},c)$, we evaluate it at a real data sample $\bm{x}$ drawn under the same condition. Fig.~\ref{fig: diff between dmd/dfd} illustrates this difference.

In practice, we do not enforce Eq.~\eqref{eq:validity condition} as a hard constraint. Instead, we satisfy it implicitly by applying DFD as a post-training stage on top of a DMD2-pretrained model, whose generations are already close to the real video and naturally keep $\E_{\bm{x}, \bm{z}}[\|\bm{x} - G_\theta(\bm{z},c)\|^{2}]$ small. The ablation in \cref{sub: ablate study} confirms the importance of this regime: violating the validity condition---e.g., applying DFD from scratch when the student still produces noisy videos---prevents the model from converging.
Concretely, we form the practical update as a combination of the DMD gradient and the DFD gradient,
\begin{equation}\label{eq: practical update gradient}
    g(\theta) \;=\; (1-w)\, g_{\mathrm{DMD}}(\theta) \;+\; w\, g_{\mathrm{DFD}}(\theta),
\end{equation}
where $w \in [0,1]$ controls how much of the DMD signal is retained. We use $w = \tfrac{1}{2}$ as our default, which preserves the fast convergence of the DMD while leveraging real-data guidance.
In practice, we implement Eq.~\ref{eq: practical update gradient} via per-step stochastic sampling, which matches in expectation when $p = w$:
\begin{equation}\label{eq:hybrid}
    \nabla_\theta \mathcal{L} =
    \begin{cases}
        g_{\mathrm{DFD}}(\theta) & \text{with probability } p, \\[4pt]
        g_{\mathrm{DMD}}(\theta) & \text{with probability } 1-p.
    \end{cases}
\end{equation}
Additionally, following standard practice in the DMD line of work, we replace the exact score in Eq.~\ref{eq: practical update gradient} with the scores estimated by the diffusion models on perturbed samples, and take the expectation over the diffusion timesteps.
\paragraph{Pseudo-Code}
We provide the pseudo-code of our DFD below, and highlight the simplicity of our method, which only adds one line of code change compared with the original DMD2 method:
\begin{figure}[ht]
\centering
\algrule[1pt]
\textbf{Algorithm 1} Pseudocode of Student Update Step in a PyTorch-like Style.
\algrule[.4pt]
\begin{lstlisting}[basicstyle=\footnotesize\ttfamily]
# student_network / teacher_network / fake_score_network: networks
# input_student, t_student: noisy input and timestep fed to student
# t, eps: diffusion timestep and noise for forward diffusion
# data: real data batch; condition: video gen conditioning
# p: probability of using DFD vs DMD 
def student_update_step(input_student, t_student, t, eps, data, condition=None):
    # Generate samples using student network.
    gen_data = student_network(input_student, t_student, condition=condition)
    
    # teacher_data = gen_data.detach() # Original DMD update
    (@\color{brightviolet}\bfseries teacher\_data = data.detach() if (torch.rand() < p) else gen\_data.detach()  @)
    
    # Inject noise into the data via forward diffusion
    perturbed_data         = forward_diffusion(gen_data, eps, t)
    perturbed_teacher_data = forward_diffusion(teacher_data, eps, t)
    # Estimate scores
    fake_score    = fake_score_network(perturbed_data, t, condition=condition)
    teacher_score = teacher_network(perturbed_teacher_data, t, condition=condition)
    # Compute gradient and student loss
    vsd_grad      = fake_score - teacher_score
    pseudo_target = gen_data - vsd_grad
    gen_loss      = 0.5 * F.mse_loss(gen_data, pseudo_target.detach())
    return gen_loss
\end{lstlisting}
\algrule[.4pt]
{The \textcolor{brightviolet}{\textbf{violet}} block is the \emph{only} difference between the DFD and DMD frameworks.}
\end{figure}
\subsection{Data Selection for DFD}
Once real data enters the distillation gradient, its quality and diversity become the new bottleneck on what the student can learn. We therefore curate the training videos with high quality and diversity. Specifically, we start from the public ViPE-Wild-1M dataset~\cite{huang2025vipe}, generated by the Wan 2.1~\cite{wan2025wan} 14B model, and cluster its 1M videos into 1{,}000 clusters. We then manually retain the clusters whose videos exhibit clear semantics, temporal coherence, varied styles, and diverse dynamics. From the retained pool, we construct two training sets: a 20{,}000-video animation and cartoon set and a complete 30{,}000-video mixed-style set. We detail the exact selection process in Appendix \ref{app: experiment detail}
\section{Experiments and Results}
\label{sec:experiments}
\begin{figure*}[t]
\begin{center}
    \includegraphics[width=1\textwidth]{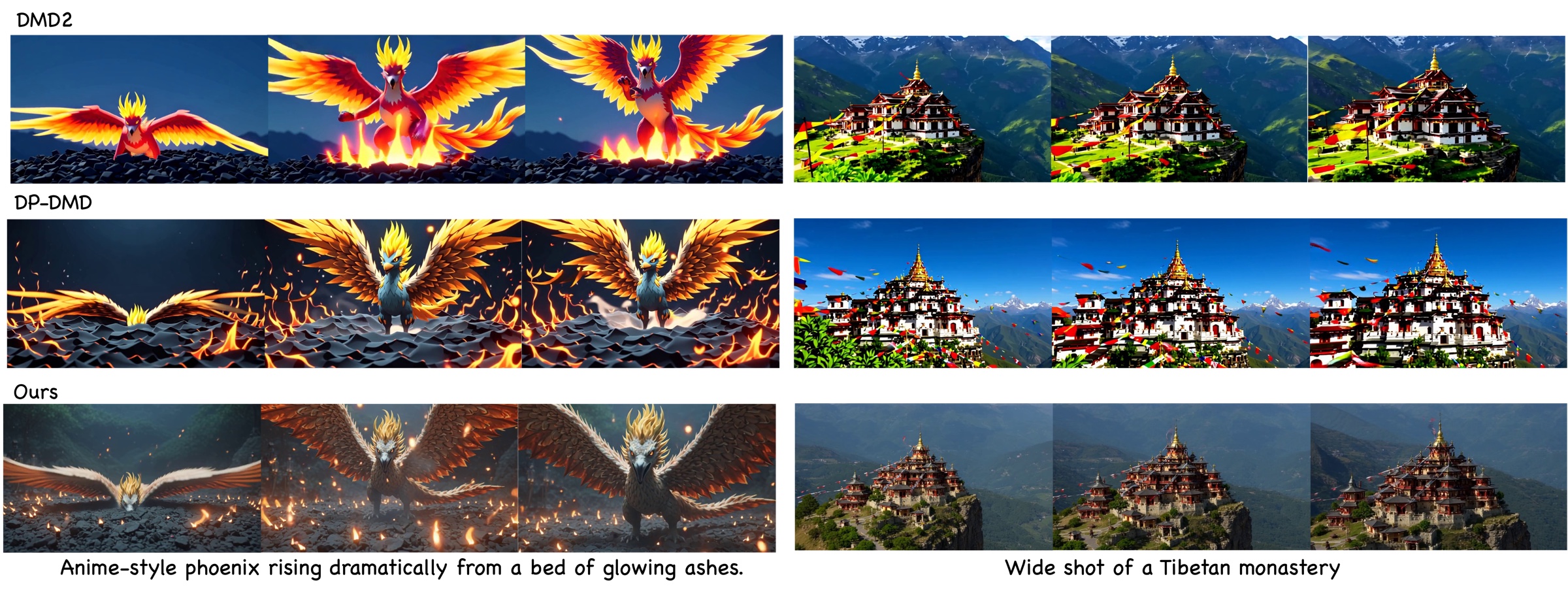}
\end{center}
\caption{ \textbf{Qualitative results on text-to-video generation.} The left columns show models distilled on animation set, and the right column on the mix-style set. Our method produces videos that are not over-saturated, and recovers finer details such as the wing texture of the phoenix.}
\label{fig:main results t2v 1}

\end{figure*}
We evaluate DFD on three representative large-scale video-generation settings: Wan2.1-1.3B for text-to-video generation and Cosmos-Predict2.5-2B for image-to-video generation,
We further conduct a comprehensive ablation study to demonstrate the effectiveness of our design choices.

\subsection{Main Experiments}

\subsubsection{Text-to-Video Generation}
\paragraph{Experimental Settings.}
For the text-to-video generation, we distill the Wan2.1-1.3B model on both the animation set and the mix-style set.
We compare against two baselines: DMD2 \cite{yin2024improved} and DP-DMD \cite{wu2026diversity}. For DMD2, we use the implementation and default configurations from the FastGen codebase. For DP-DMD, we follow the original paper and use a diversity anchor step of $K=5$ with weight $0.05$ \cite{wu2026diversity}. For all the methods, we distill the pretrained Wan model into a four-step student model. 
We evaluate results along three aspects. \textbf{Video quality}: we adopt the metrics from VBench \cite{huang2024vbench}. \textbf{Video diversity}: we adopt the diversity metric from DP-DMD \cite{wu2026diversity}. \textbf{Camera-pose diversity}: we use the ViPE \cite{huang2025vipe} to estimate the camera poses of each generated video, and compute diversity statistics from the estimated cameras. 
The evaluation set comprises 70 animation prompts and 70 mix-style prompts, respectively. We generate 8 videos for each prompt with random seeds $1$--$8$, yielding $1120$ videos in total for evaluation.
Further details are provided in the Appendix~\ref{app: experiment detail}.

\begin{table*}[h]
  \centering
  \caption{\footnotesize Quantitative results on the text-to-video experiments using Wan2.1-1.3B. For all metrics, higher is better. The best result among the distilled methods in each column is \textbf{bolded}.}
  \label{tab:results_v3}
  \setlength{\tabcolsep}{4pt}
  \renewcommand{\arraystretch}{1.2}
  \resizebox{\textwidth}{!}{%
    \begin{tabular}{l cccccc cccc cc}
      \toprule
      & \multicolumn{6}{c}{\textbf{Video Quality}}
      & \multicolumn{4}{c}{\textbf{Video Diversity}}
      & \multicolumn{2}{c}{\textbf{Camera Pose Diversity}} \\
      \cmidrule(lr){2-7} \cmidrule(lr){8-11} \cmidrule(lr){12-13}
      \textbf{Method} &
      \shortstack{\textbf{Subject}\\\textbf{Consistency}} &
      \shortstack{\textbf{Background}\\\textbf{Consistency}} &
      \shortstack{\textbf{Temporal}\\\textbf{Flickering}} &
      \shortstack{\textbf{Motion}\\\textbf{Smoothness}} &
      \shortstack{\textbf{Aesthetic}\\\textbf{Quality}} &
      \shortstack{\textbf{Average}\\\textbf{VBench}} &
      \shortstack{\textbf{CLIP}\\\textbf{(Mean)}} &
      \shortstack{\textbf{CLIP}\\\textbf{(Per-frame)}} &
      \shortstack{\textbf{DINO}\\\textbf{(Mean)}} &
      \shortstack{\textbf{DINO}\\\textbf{(Per-frame)}} &
      \shortstack{\textbf{Endpoint}\\\textbf{Distance}} &
      \shortstack{\textbf{Trajectory}\\\textbf{Distance}}
      \\
      \midrule
      Teacher & 0.956 & 0.959 & 0.976 & 0.985 & 0.622 & 0.899 & 0.178 & 0.222 & 0.301 & 0.350 & 30.571 & 26.651 \\
      \midrule
      DMD2  & 0.956 & \textbf{0.957} & 0.973 & 0.985 & 0.634 & 0.901 & 0.120 & 0.165 & 0.190 & 0.239 & 9.148 & 4.284\\
      DP-DMD & \textbf{0.960} & \textbf{0.957} & \textbf{0.977} &0.987 & 0.633 & 0.903 & 0.126 & 0.165 & 0.197 & 0.239 & 7.208 & 3.466 \\
      Ours  & 0.956 & 0.955 & 0.976 & \textbf{0.988} & \textbf{0.655} & \textbf{0.906} & \textbf{0.128} & \textbf{0.170} & \textbf{0.205} & \textbf{0.252} & \textbf{18.513} & \textbf{19.256} \\
    
      \bottomrule
    \end{tabular}%
  }
\end{table*}


\paragraph{Experimental Results}
We provide qualitative results in Fig.~\ref{fig:main results t2v 1} with quantitative results averaged on the models distilled on the animation test set and the mix-style testset in Table~\ref{tab:results_v3}. 
Qualitatively, our method produces higher-quality videos: it mitigates the over-saturation artifacts with significantly better appearance and dynamics, and recovers fine-grained details.
Quantitatively, DFD obtains the best overall VBench score among all the distilled models, mainly driven by improved aesthetic quality and motion smoothness. It also improves all four visual-diversity metrics and substantially increases camera-pose diversity compared with DMD2 and DP-DMD.
Additional qualitative results are provided in Appendix \ref{app sub: additional results for i2v}.

%


\subsubsection{Image-to-Video Generation}
\paragraph{Experimental Settings.}
For image-to-video generation, we distill the Cosmos-Predict2.5-2B model on our curated mix-style dataset.
Same as text-to-video generation, all the methods are implemented in the FastGen codebase.
For evaluation, we use two image-to-video test sets: one from VBench, containing 348 images, and one curated from the ViPE-Wild-1M dataset, containing 78 images that are held out from training. We
adopt the image-to-video metrics from VBench \cite{huang2024vbench} for evaluation. Further details are provided in the Appendix~\ref{app: experiment detail}.

\paragraph{Experimental Results.}
We provide quantitative results in Table~\ref{tab:vbench-comparison-vbench} and Table~\ref{tab:vbench-comparison-vipe} (in the appendix), with qualitative results in Fig~\ref{fig:main results i2v vipe} and Fig~\ref{fig:main results i2v vbench}. Quantitatively, our method outperforms the baselines on the majority of metrics, and is particularly strong at preserving first-frame conditioning and maintaining temporal coherence. Qualitatively, DP-DMD frequently violates the first-frame constraint, occasionally producing abrupt artifacts that disrupt the whole frame
DMD2 yields visible quality improvements over DP-DMD but still struggles to maintain temporal coherence in later frames and often hallucinates structural anomalies such as distorted human figures appearing outside vehicles. 
In contrast, our model adheres closely to the input frame and remains stable and consistent across all generated frames.
Our model also achieves better physical plausibility.
In the traffic scene example shown in Fig.~\ref{fig:main results i2v vbench}, DMD2 produces visibly implausible interactions, such as vehicles intersecting. DFD reduces such artifacts in this example and yields more coherent motion.
These results support the core intuition of our approach: by incorporating real data into the distillation process, the generator can not only mitigate mode-seeking behavior and over-saturation artifacts, but also capture key properties of real videos, such as first-frame coherence, temporal consistency, and physical realism.
\begin{figure*}[t]  
\begin{center}
    \includegraphics[width=1\textwidth]{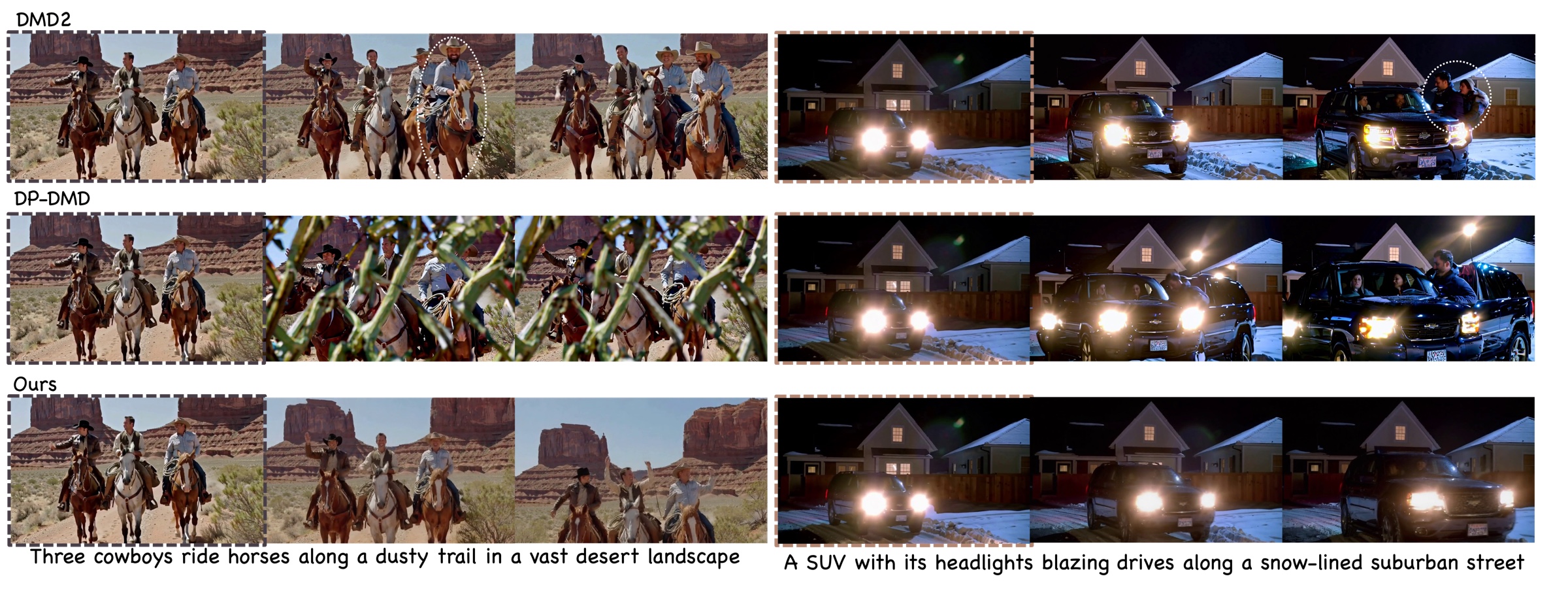}
\end{center}
\caption{ \textbf{I2V results on the ViPE test set.} The colored outline indicates the input image. Our method produces videos that closely follow the first frame and remain coherent across the full sequence, whereas DMD2 and DP-DMD introduce structural anomalies such as an extra cowboy appearing from nowhere (highlighted by white circles).}
\label{fig:main results i2v vipe}
\end{figure*}
\begin{table*}[h]
\centering
\caption{ Quantitative comparison on image-to-video generation, we evaluate on the VBench test image suite using the metrics from VBench. For all metrics, higher is better. The best result in each column is \textbf{bolded}. Our method consistently outperforms the baselines on all the metrics, and is even better than the teacher model.}
\setlength{\tabcolsep}{4pt}
\renewcommand{\arraystretch}{1.2}
\resizebox{\textwidth}{!}{%
\begin{tabular}{l|ccccccc|c}
\toprule
\textbf{Method} &
\shortstack{\textbf{Subject}\\\textbf{Consistency}} &
\shortstack{\textbf{Background}\\\textbf{Consistency}} &
\shortstack{\textbf{Aesthetic}\\\textbf{Quality}} &
\shortstack{\textbf{Temporal}\\\textbf{Flickering}} &
\shortstack{\textbf{Motion}\\\textbf{Smoothness}} &
\shortstack{\textbf{Image-to-Video}\\\textbf{Subject}} &
\shortstack{\textbf{Image-to-Video}\\\textbf{Background}} &
\textbf{Average} \\
\midrule
Teacher       & 0.9616 & 0.9648 & 0.6320 & 0.9717 & 0.9905 & 0.9867 & 0.9925 & 0.9285 \\
\midrule
DMD2          & 0.9421 & 0.9621 & 0.6301 & 0.9734 & 0.9886 & 0.9850 & 0.9900 & 0.9245 \\
DP-DMD        & 0.8457 & 0.9068 & 0.5850 & 0.9603 & 0.9832 & 0.9570 & 0.9520 & 0.8843 \\
Ours & \textbf{0.9613} & \textbf{0.9692} & \textbf{0.6371} & \textbf{0.9759} & \textbf{0.9900} & \textbf{0.9859} & \textbf{0.9930} & \textbf{0.9303} \\
\bottomrule
\end{tabular}%
}
\label{tab:vbench-comparison-vbench}
\end{table*}

\begin{figure*}[htbp]
\begin{center}
    \includegraphics[width=1\textwidth]{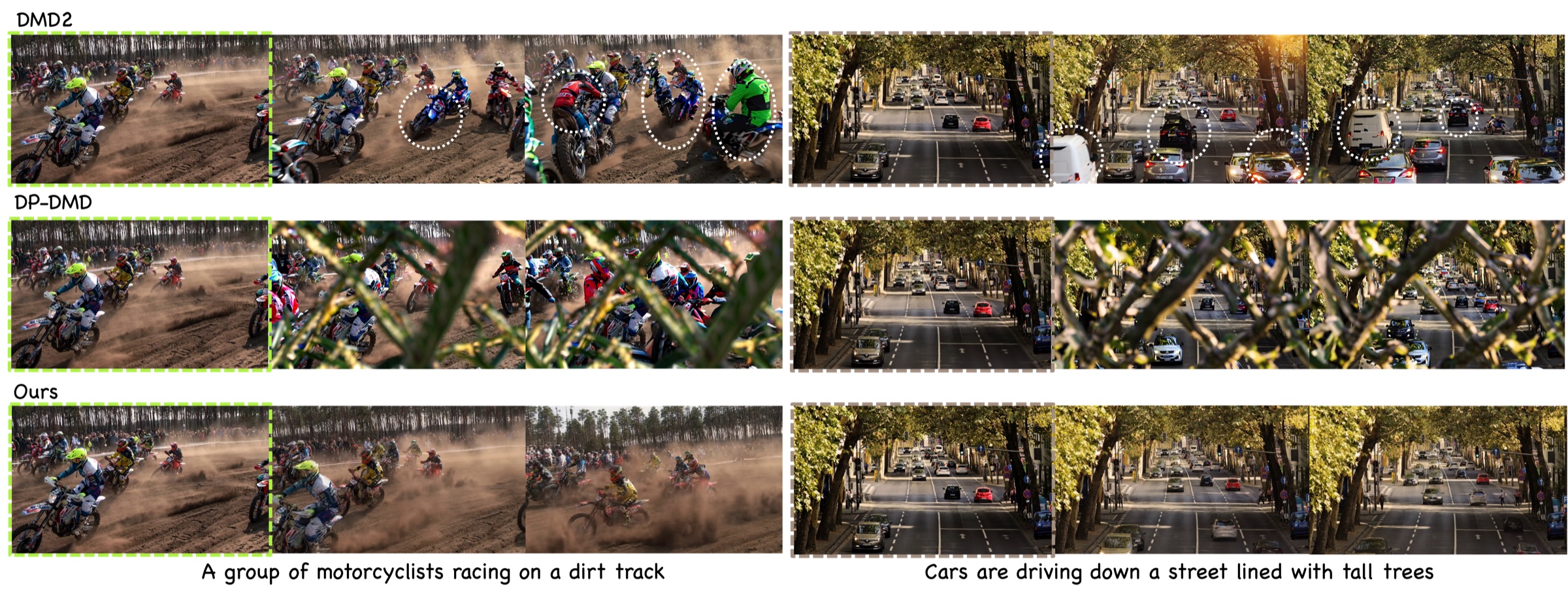}
\end{center}
\caption{\textbf{I2V results on the VBench test set.} The colored outline indicates the input image. Our method shows much better visual quality with smooth and physically plausible dynamics under complex scenarios, whereas DP-DMD and DMD2 fail to produce valid videos.}
\label{fig:main results i2v vbench}
\end{figure*}
\subsubsection{Autoregressive Video Generation}
\paragraph{Experimental Settings.}
For autoregressive video generation, we distill the Wan2.1-1.3B model on our curated mixed-style dataset following the \textit{Self Forcing} framework~\cite{huang2025selfforcing}. All methods are implemented within the Self Forcing codebase.
Additional implementation details are provided in Appendix~\ref{app: experiment detail}.

\paragraph{Experimental Results.}
Qualitative results are presented in Fig.~\ref{fig:self_forcing}. 
Similar to our image-to-video experiments, Self Forcing with DMD2 struggles to preserve temporal coherence over long horizons and frequently introduces structural artifacts and hallucinations. For example, crab-like figures gradually emerge from the stone scene.
In contrast, our method maintains stable content and coherent object structures throughout the generated sequence.
Our method also produces videos with stronger physical plausibility. 
In the train example shown in Fig.~\ref{fig:self_forcing}, DMD2 generates physically inconsistent interactions including disconnected train carriages.
DFD substantially reduces such artifacts and yields a more coherent and realistic scene. 
These observations echo our I2V findings and further support the core intuition behind our approach: 
by incorporating real data into the distillation process, 
the generator not only alleviates mode-seeking behavior and over-saturation artifacts, but also encourages the model to capture key characteristics of real videos,
such as inter-frame coherence, temporal consistency, and physical realism.
\begin{figure*}[htbp]
\begin{center}
    \includegraphics[width=1\textwidth]{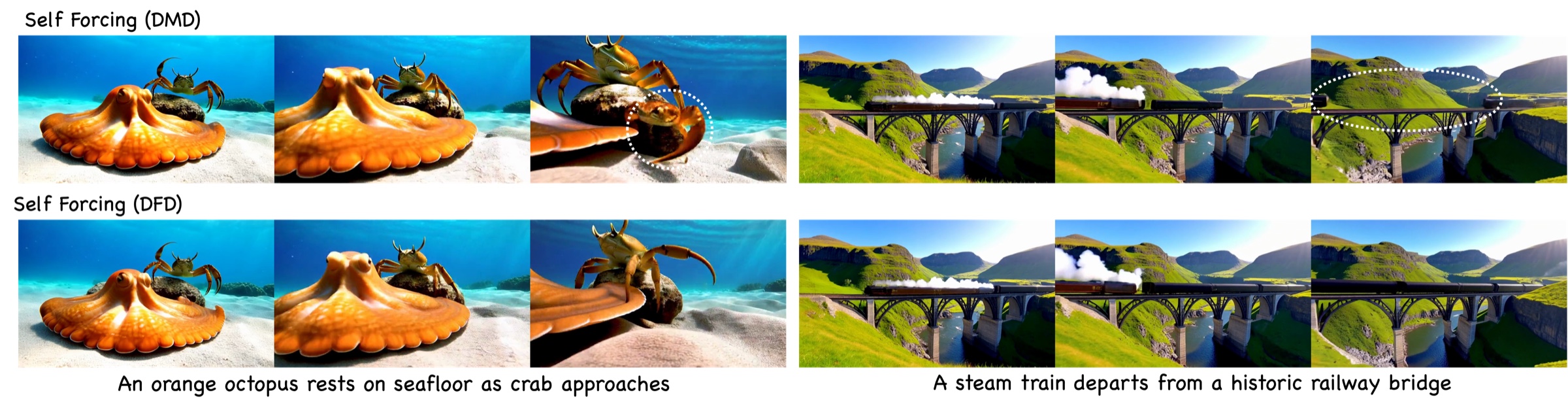}
\end{center}
\caption{\textbf{Autoregressive video generation results.} Results from the generator trained with the Self Forcing pipeline using DFD loss outperform those trained with DMD loss, especially in frame consistency and physical plausibility.}
\label{fig:self_forcing}
\end{figure*}

\begin{figure*}[h]
\begin{center}
    \includegraphics[width=1\textwidth]{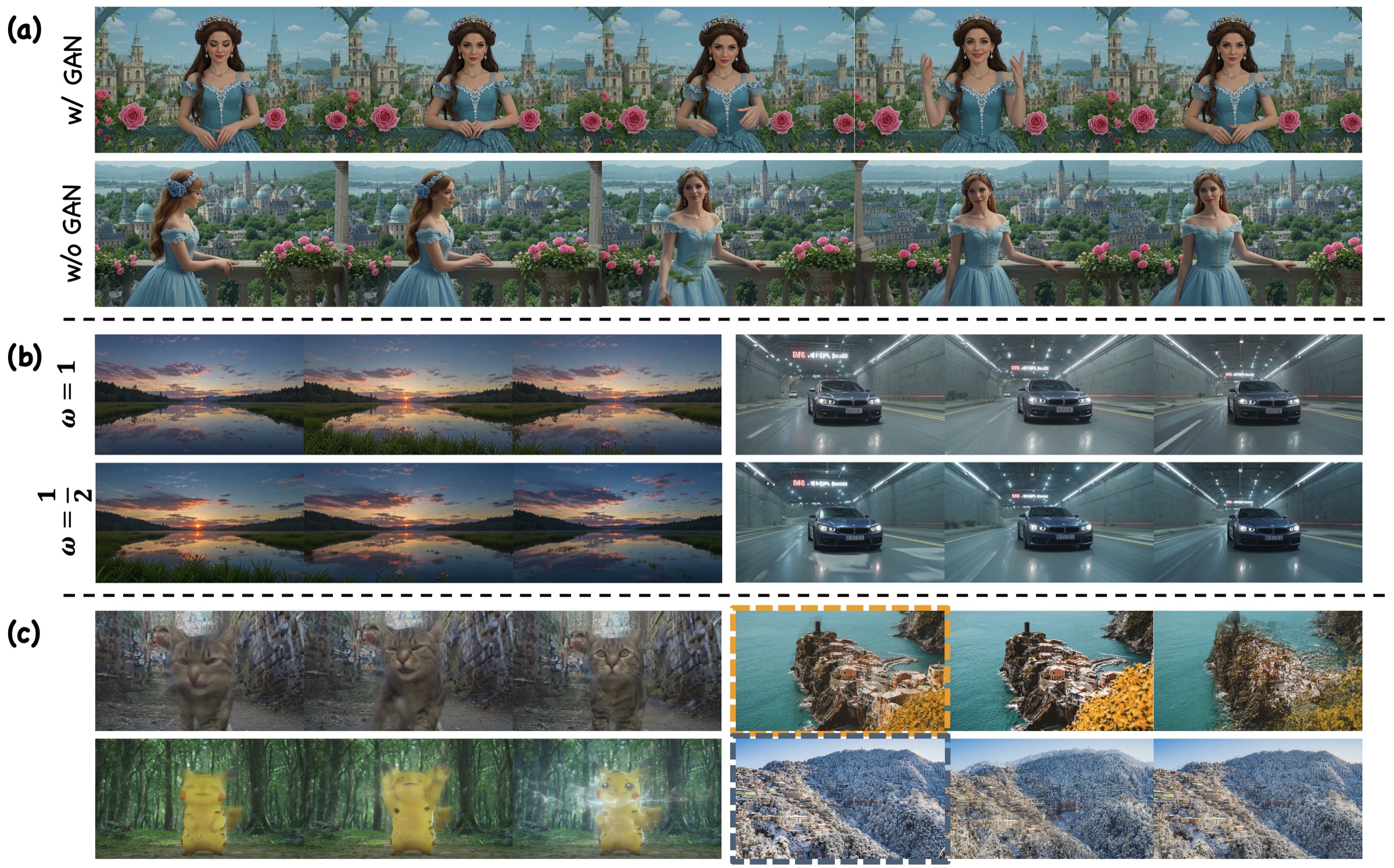}
\end{center}
\caption{\textbf{(a): Ablation on the effect of the GAN loss.} Adding the GAN loss yields no clear quality improvement, and video dynamics even degrade, which is consistent with our quantitative results. \textbf{(b): Ablation on the weight in Eq.~\ref{eq: practical update gradient}}: $w=1$ (upper) vs.\ $w=\frac{1}{2}$ (lower). Each row shows three evenly sampled frames from one generated video. There is no clear visual difference between the two settings. \textbf{(c): Qualitative results without DMD2 pretraining.} Initializing purely from the teacher model without DMD2 pretraining fails to produce reasonable results, even after a sufficient number of training steps for both the text-to-video and image-to-video tasks.}
\label{fig:train_from_scratch}
\end{figure*}

\begin{figure*}[h]  
\begin{center}
    \includegraphics[width=1\textwidth]{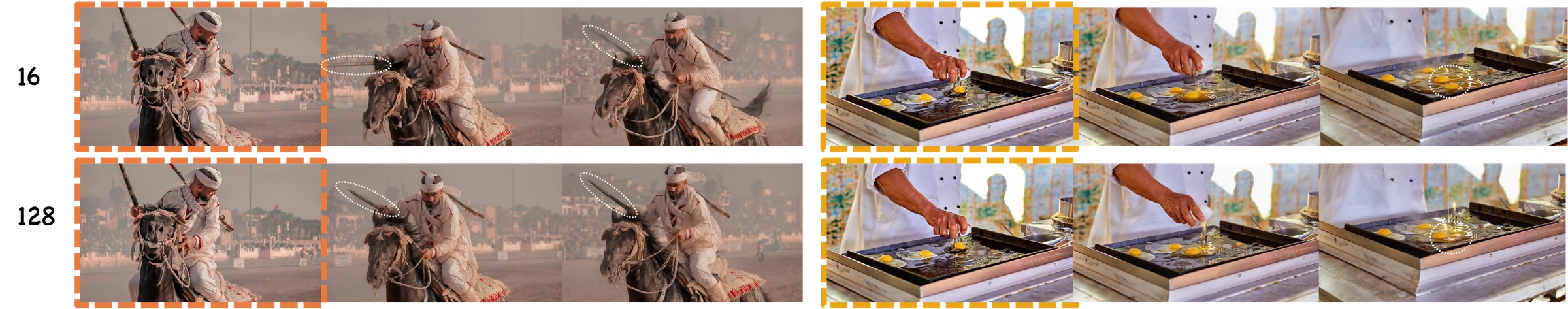}
\end{center}
\caption{ \textbf{Effect of scaling up the distillation batch size.} The colored outline indicates the input image. Increasing the batch size from 16 to 128 yields videos with superior visual quality and physical consistency. For instance, the spear in the rider's hand exhibits much greater temporal clarity. Additionally, the scaled model preserves exact object counts across frames (the eggs), whereas the smaller batch size struggles with physical plausibility by generating phantom objects out of nowhere (highlighted by the white circle).}
\label{fig:scale_up}

\end{figure*}

\subsection{Ablation Study}
\label{sub: ablate study}



\begin{table*}[h]
\centering
\caption{Comparison between models distilled with batch size $16$ and $128$, evaluated on the VBench test set.}
\setlength{\tabcolsep}{4pt}
\renewcommand{\arraystretch}{1.2}
\resizebox{\textwidth}{!}{%
\begin{tabular}{l|ccccccc|c}
\toprule
\textbf{Batch size} &
\shortstack{\textbf{Subject}\\\textbf{Consistency}} &
\shortstack{\textbf{Background}\\\textbf{Consistency}} &
\shortstack{\textbf{Aesthetic}\\\textbf{Quality}} &
\shortstack{\textbf{Temporal}\\\textbf{Flickering}} &
\shortstack{\textbf{Motion}\\\textbf{Smoothness}} &
\shortstack{\textbf{Image-to-Video}\\\textbf{Subject}} &
\shortstack{\textbf{Image-to-Video}\\\textbf{Background}} &
\textbf{Average} \\

\midrule
$16$          & 0.9613 & \textbf{0.9692} & 0.6371 & \textbf{0.9759} & 0.9900 & 0.9859 & \textbf{0.9930} & 0.9303 \\
$128$        & \textbf{0.9638} & 0.9685 & \textbf{0.6383} & 0.9783 & \textbf{0.9914} & \textbf{0.9878} & 0.9929 & \textbf{0.9316} \\

\bottomrule
\end{tabular}%
}
\label{tab:scale up}
\end{table*}

\paragraph{Effects of the GAN Loss.}
We ablate the requirement of GAN loss, and compare our DFD pipeline with and without an additional GAN loss. We experiment on the text-to-video generation on animation set. The VBench quality scores are reported in Table~\ref{tab: ablate gan} (in the appendix), and qualitative results in Fig~\ref{fig:train_from_scratch} (a). 
This ablation suggests that, in our post-training setting, adding the GAN loss does not provide a consistent benefit. Removing it simplifies the pipeline and improves dynamic degree and imaging quality, while maintaining comparable scores on the remaining metrics.
\paragraph{Effects on Weights in Eq.~\ref{eq: practical update gradient}.}
We ablate the weights in the Eq.~\ref{eq: practical update gradient}. We experiment on the text-to-video generation on animation set. We compare two weighting schemes, $w=\frac{1}{2}$ (our default setting) and $w=1$ (pure DFD without DMD), on the text-to-video generation. Quantitative results are reported in Table~\ref{tab:hybrid_comparison} (in the appendix), and qualitative results in Fig~\ref{fig:train_from_scratch} (b). The differences between the two settings are small, indicating that our method is stable and insensitive to the choice of $w$. We adopt $w=\frac{1}{2}$ for theoretical reasons as it better satisfies the condition in Eq.~\ref{eq:validity condition}.
\paragraph{Effects of DMD2 Pretraining.}
Our DFD needs to satisfy the \emph{validity condition} in Eq.~\ref{eq:validity condition} for stabilized training. To verify this, we ablate by initializing the student with the teacher model (without DMD2 distillation) and comparing the DMD2 initialization. We experiment on both the text-to-video with animation set and image-to-video tasks with mix-style set.
Qualitative results are provided in Fig~\ref{fig:train_from_scratch} (c). Even after a long training run (e.g., 1400 iterations), the model fails to converge, which underscores the necessity of the validity condition.

\subsection{Scaling Up with Large Batch Size}
We further scale up the image-to-video generation on the Cosmos-Predict2.5-2B model by increasing the batch size from $16$ to $128$. As shown quantitatively in Table~\ref{tab:scale up} and visually in Fig~\ref{fig:scale_up}, the model distilled with a large batch size consistently outperforms the model with a small batch size. The larger batch size yields remarkably better temporal stability during large motions (e.g., the rider's spear) and improved physical plausibility, such as preserving object permanence (e.g., consistent egg counts).

\section{Conclusion}
\label{sec:conclusion}
We introduce a data forcing distillation framework that advances the state-of-the-art for few-step video diffusion models, and restores diversity and fidelity. Our DFD requires only a single line of code change to the DMD2 baseline. At its core is the explicit integration of real data into the distribution matching objective, which successfully overcomes the diversity degradation and over-saturation typically induced by reverse-KL formulations. We provide a theoretical analysis to understand our DFD. Through comprehensive, large-scale evaluations on text-to-video and image-to-video generation benchmarks, we demonstrate that our approach successfully resolves persistent failure modes in DMD, yielding substantial improvements in visual fidelity, temporal coherence, and physical plausibility in few-step video generation. We also apply our method to the autoregressive video generation settings under the \textit{Self Forcing} frameworks \cite{huang2025selfforcing}, prelimary results show our method still show the same improvements over the DMD method

\section{Limitation}
\label{sec:limitation}
While DFD substantially improves few-step video generation, its performance remains limited under extremely constrained generation budget. 
As shown in \cref{fig:limitation}, when the generation is restricted to two steps or fewer, 
the two-step distilled Cosmos-Predict2.5-2B model still struggles to produce high-quality videos: fast-moving objects, such as the man's hands, appear blurry; fine details, such as the woman's facial features, lose fidelity; and the model sometimes collapses to nearly static videos with little motion. 
These results suggest that, despite the improvements brought by DFD, generating high-quality and temporally dynamic videos remains challenging in the highly aggressive two-step-or-fewer regime.

\begin{figure*}[htbp]
\begin{center}
    \includegraphics[width=1\textwidth]{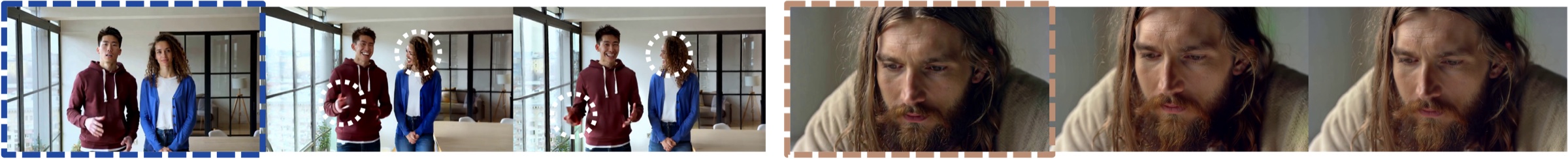}
\end{center}
\caption{\textbf{Limitation of DFD under a two-step generation budget.} With two-step distillation of the Cosmos-Predict2.5-2B model, DFD still produces blurry results for fast-moving content (e.g., the man's hands), loses fine detail (e.g., the woman's face), or collapses to highly static videos.}
\label{fig:limitation}
\end{figure*}

\bibliographystyle{unsrt}
\bibliography{reference}

@article{poole2022dreamfusion,
  title={Dreamfusion: Text-to-3d using 2d diffusion},
  author={Poole, Ben and Jain, Ajay and Barron, Jonathan T and Mildenhall, Ben},
  journal={arXiv preprint arXiv:2209.14988},
  year={2022}
}

@article{wang2023prolificdreamer,
  title={Prolificdreamer: High-fidelity and diverse text-to-3d generation with variational score distillation},
  author={Wang, Zhengyi and Lu, Cheng and Wang, Yikai and Bao, Fan and Li, Chongxuan and Su, Hang and Zhu, Jun},
  journal={Advances in neural information processing systems},
  volume={36},
  pages={8406--8441},
  year={2023}
}

@article{wan2025wan,
  title={Wan: Open and advanced large-scale video generative models},
  author={Wan, Team and Wang, Ang and Ai, Baole and Wen, Bin and Mao, Chaojie and Xie, Chen-Wei and Chen, Di and Yu, Feiwu and Zhao, Haiming and Yang, Jianxiao and others},
  journal={arXiv preprint arXiv:2503.20314},
  year={2025}
}

@inproceedings{huang2025vipe,
    title={ViPE: Video Pose Engine for 3D Geometric Perception},
    author={Huang, Jiahui and Zhou, Qunjie and Rabeti, Hesam and Korovko, Aleksandr and Ling, Huan and Ren, Xuanchi and Shen, Tianchang and Gao, Jun and Slepichev, Dmitry and Lin, Chen-Hsuan and Ren, Jiawei and Xie, Kevin and Biswas, Joydeep and Leal-Taixe, Laura and Fidler, Sanja},
    booktitle={NVIDIA Research Whitepapers arXiv:2508.10934},
    year={2025}
}

@article{yin2024improved,
  title={Improved distribution matching distillation for fast image synthesis},
  author={Yin, Tianwei and Gharbi, Micha{\"e}l and Park, Taesung and Zhang, Richard and Shechtman, Eli and Durand, Fredo and Freeman, William T},
  journal={Advances in neural information processing systems},
  volume={37},
  pages={47455--47487},
  year={2024}
}

@article{wu2026diversity,
  title={Diversity-Preserved Distribution Matching Distillation for Fast Visual Synthesis},
  author={Wu, Tianhe and Li, Ruibin and Zhang, Lei and Ma, Kede},
  journal={arXiv preprint arXiv:2602.03139},
  year={2026}
}

@inproceedings{huang2024vbench,
  title={Vbench: Comprehensive benchmark suite for video generative models},
  author={Huang, Ziqi and He, Yinan and Yu, Jiashuo and Zhang, Fan and Si, Chenyang and Jiang, Yuming and Zhang, Yuanhan and Wu, Tianxing and Jin, Qingyang and Chanpaisit, Nattapol and others},
  booktitle={Proceedings of the IEEE/CVF Conference on Computer Vision and Pattern Recognition},
  pages={21807--21818},
  year={2024}
}

@article{song2020score,
  title={Score-based generative modeling through stochastic differential equations},
  author={Song, Yang and Sohl-Dickstein, Jascha and Kingma, Diederik P and Kumar, Abhishek and Ermon, Stefano and Poole, Ben},
  journal={arXiv preprint arXiv:2011.13456},
  year={2020}
}

@article{lipman2022flow,
  title={Flow matching for generative modeling},
  author={Lipman, Yaron and Chen, Ricky TQ and Ben-Hamu, Heli and Nickel, Maximilian and Le, Matt},
  journal={arXiv preprint arXiv:2210.02747},
  year={2022}
}

@inproceedings{rombach2022high,
  title={High-resolution image synthesis with latent diffusion models},
  author={Rombach, Robin and Blattmann, Andreas and Lorenz, Dominik and Esser, Patrick and Ommer, Bj{\"o}rn},
  booktitle={Proceedings of the IEEE/CVF conference on computer vision and pattern recognition},
  pages={10684--10695},
  year={2022}
}

@article{cosmos-predict2p5,
  author  = {Ali, Arslan and Bai, Junjie and Bala, Maciej and Balaji, Yogesh and Blakeman, Aaron and Cai, Tiffany and Cao, Jiaxin and Cao, Tianshi and Cha, Elizabeth and Chao, Yu-Wei and other},
  title   = {World Simulation with Video Foundation Models for Physical {AI}},
  journal = {arXiv preprint arXiv:2511.00062},
  year    = {2025},
}

@article{kong2024hunyuanvideo,
  title={Hunyuanvideo: A systematic framework for large video generative models},
  author={Kong, Weijie and Tian, Qi and Zhang, Zijian and Min, Rox and Dai, Zuozhuo and Zhou, Jin and Xiong, Jiangfeng and Li, Xin and Wu, Bo and Zhang, Jianwei and others},
  journal={arXiv preprint arXiv:2412.03603},
  year={2024}
}

@article{yang2024cogvideox,
  title={Cogvideox: Text-to-video diffusion models with an expert transformer},
  author={Yang, Zhuoyi and Teng, Jiayan and Zheng, Wendi and Ding, Ming and Huang, Shiyu and Xu, Jiazheng and Yang, Yuanming and Hong, Wenyi and Zhang, Xiaohan and Feng, Guanyu and others},
  journal={arXiv preprint arXiv:2408.06072},
  year={2024}
}

@article{brooks2024video,
  title={Video generation models as world simulators},
  author={Brooks, Tim and Peebles, Bill and Holmes, Connor and DePue, Will and Guo, Yufei and Jing, Leo and Schnurr, David and Taylor, Joe and Luhman, Troy and Luhman, Eric and others},
  journal={OpenAI Blog},
  volume={1},
  number={8},
  pages={1},
  year={2024}
}

@article{salimans2022progressive,
  title={Progressive distillation for fast sampling of diffusion models},
  author={Salimans, Tim and Ho, Jonathan},
  journal={arXiv preprint arXiv:2202.00512},
  year={2022}
}

@article{song2023consistency,
  title={Consistency models},
  author={Song, Yang and Dhariwal, Prafulla and Chen, Mark and Sutskever, Ilya},
  year={2023}
}

@inproceedings{yin2024one,
  title={One-step diffusion with distribution matching distillation},
  author={Yin, Tianwei and Gharbi, Micha{\"e}l and Zhang, Richard and Shechtman, Eli and Durand, Fredo and Freeman, William T and Park, Taesung},
  booktitle={Proceedings of the IEEE/CVF conference on computer vision and pattern recognition},
  pages={6613--6623},
  year={2024}
}

@article{luo2023diff,
  title={Diff-instruct: A universal approach for transferring knowledge from pre-trained diffusion models},
  author={Luo, Weijian and Hu, Tianyang and Zhang, Shifeng and Sun, Jiacheng and Li, Zhenguo and Zhang, Zhihua},
  journal={Advances in Neural Information Processing Systems},
  volume={36},
  pages={76525--76546},
  year={2023}
}

@article{wang2025uni,
  title={Uni-instruct: One-step diffusion model through unified diffusion divergence instruction},
  author={Wang, Yifei and Bai, Weimin and Zhang, Colin and Zhang, Debing and Luo, Weijian and Sun, He},
  journal={arXiv preprint arXiv:2505.20755},
  year={2025}
}

@article{zheng2025large,
  title={Large scale diffusion distillation via score-regularized continuous-time consistency},
  author={Zheng, Kaiwen and Wang, Yuji and Ma, Qianli and Chen, Huayu and Zhang, Jintao and Balaji, Yogesh and Chen, Jianfei and Liu, Ming-Yu and Zhu, Jun and Zhang, Qinsheng},
  journal={arXiv preprint arXiv:2510.08431},
  year={2025}
}

@article{luo2023latent,
  title={Latent consistency models: Synthesizing high-resolution images with few-step inference},
  author={Luo, Simian and Tan, Yiqin and Huang, Longbo and Li, Jian and Zhao, Hang},
  journal={arXiv preprint arXiv:2310.04378},
  year={2023}
}

@inproceedings{meng2023distillation,
  title={On distillation of guided diffusion models},
  author={Meng, Chenlin and Rombach, Robin and Gao, Ruiqi and Kingma, Diederik and Ermon, Stefano and Ho, Jonathan and Salimans, Tim},
  booktitle={Proceedings of the IEEE/CVF conference on computer vision and pattern recognition},
  pages={14297--14306},
  year={2023}
}

@article{li2024t2v,
  title={T2v-turbo: Breaking the quality bottleneck of video consistency model with mixed reward feedback},
  author={Li, Jiachen and Feng, Weixi and Fu, Tsu-Jui and Wang, Xinyi and Basu, Sugato and Chen, Wenhu and Wang, William Y},
  journal={Advances in neural information processing systems},
  volume={37},
  pages={75692--75726},
  year={2024}
}

@article{song2023improved,
  title={Improved techniques for training consistency models},
  author={Song, Yang and Dhariwal, Prafulla},
  journal={arXiv preprint arXiv:2310.14189},
  year={2023}
}

@article{lu2024simplifying,
  title={Simplifying, stabilizing and scaling continuous-time consistency models},
  author={Lu, Cheng and Song, Yang},
  journal={arXiv preprint arXiv:2410.11081},
  year={2024}
}

@inproceedings{geng2024consistency,
  title={Consistency models made easy},
  author={Geng, Zhengyang and Pokle, Ashwini and Luo, Weijian and Lin, Justin and Kolter, J Zico},
  booktitle={The Thirteenth International Conference on Learning Representations},
  year={2024}
}

@article{geng2025mean,
  title={Mean flows for one-step generative modeling},
  author={Geng, Zhengyang and Deng, Mingyang and Bai, Xingjian and Kolter, J Zico and He, Kaiming},
  journal={arXiv preprint arXiv:2505.13447},
  year={2025}
}

@article{sabour2025align,
  title={Align your flow: Scaling continuous-time flow map distillation},
  author={Sabour, Amirmojtaba and Fidler, Sanja and Kreis, Karsten},
  journal={arXiv preprint arXiv:2506.14603},
  year={2025}
}

@article{kim2023consistency,
  title={Consistency trajectory models: Learning probability flow ode trajectory of diffusion},
  author={Kim, Dongjun and Lai, Chieh-Hsin and Liao, Wei-Hsiang and Murata, Naoki and Takida, Yuhta and Uesaka, Toshimitsu and He, Yutong and Mitsufuji, Yuki and Ermon, Stefano},
  journal={arXiv preprint arXiv:2310.02279},
  year={2023}
}

@article{wang2024phased,
  title={Phased consistency models},
  author={Wang, Fu-Yun and Huang, Zhaoyang and Bergman, Alexander W and Shen, Dazhong and Gao, Peng and Lingelbach, Michael and Sun, Keqiang and Bian, Weikang and Song, Guanglu and Liu, Yu and others},
  journal={Advances in neural information processing systems},
  volume={37},
  pages={83951--84009},
  year={2024}
}

@inproceedings{zhou2024score,
  title={Score identity distillation: Exponentially fast distillation of pretrained diffusion models for one-step generation},
  author={Zhou, Mingyuan and Zheng, Huangjie and Wang, Zhendong and Yin, Mingzhang and Huang, Hai},
  booktitle={Forty-first International Conference on Machine Learning},
  year={2024}
}

@article{xu2025one,
  title={One-step Diffusion Models with $ f $-Divergence Distribution Matching},
  author={Xu, Yilun and Nie, Weili and Vahdat, Arash},
  journal={arXiv preprint arXiv:2502.15681},
  year={2025}
}

@article{simeoni2025dinov3,
  title={Dinov3},
  author={Sim{\'e}oni, Oriane and Vo, Huy V and Seitzer, Maximilian and Baldassarre, Federico and Oquab, Maxime and Jose, Cijo and Khalidov, Vasil and Szafraniec, Marc and Yi, Seungeun and Ramamonjisoa, Micha{\"e}l and others},
  journal={arXiv preprint arXiv:2508.10104},
  year={2025}
}

@inproceedings{radford2021learning,
  title={Learning transferable visual models from natural language supervision},
  author={Radford, Alec and Kim, Jong Wook and Hallacy, Chris and Ramesh, Aditya and Goh, Gabriel and Agarwal, Sandhini and Sastry, Girish and Askell, Amanda and Mishkin, Pamela and Clark, Jack and others},
  booktitle={International conference on machine learning},
  pages={8748--8763},
  year={2021},
  organization={PmLR}
}

@article{luo2024one,
  title={One-step diffusion distillation through score implicit matching},
  author={Luo, Weijian and Huang, Zemin and Geng, Zhengyang and Kolter, J Zico and Qi, Guo-jun},
  journal={Advances in Neural Information Processing Systems},
  volume={37},
  pages={115377--115408},
  year={2024}
}

@article{Nie2026TransitionMD,
  title={Transition Matching Distillation for Fast Video Generation},
  author={Weili Nie and Julius Berner and Nanye Ma and Chao Liu and Saining Xie and Arash Vahdat},
  journal={ArXiv},
  year={2026},
  volume={abs/2601.09881},
  url={https://api.semanticscholar.org/CorpusID:284738120}
}

@article{goodfellow2014generative,
  title={Generative adversarial nets},
  author={Goodfellow, Ian J and Pouget-Abadie, Jean and Mirza, Mehdi and Xu, Bing and Warde-Farley, David and Ozair, Sherjil and Courville, Aaron and Bengio, Yoshua},
  journal={Advances in neural information processing systems},
  volume={27},
  year={2014}
}

@article{he2024training,
  title={Training neural samplers with reverse diffusive kl divergence},
  author={He, Jiajun and Chen, Wenlin and Zhang, Mingtian and Barber, David and Hern{\'a}ndez-Lobato, Jos{\'e} Miguel},
  journal={arXiv preprint arXiv:2410.12456},
  year={2024}
}

@inproceedings{lu2025adversarial,
  title={Adversarial distribution matching for diffusion distillation towards efficient image and video synthesis},
  author={Lu, Yanzuo and Ren, Yuxi and Xia, Xin and Lin, Shanchuan and Wang, Xing and Xiao, Xuefeng and Ma, Andy J and Xie, Xiaohua and Lai, Jian-Huang},
  booktitle={Proceedings of the IEEE/CVF International Conference on Computer Vision},
  pages={16818--16829},
  year={2025}
}

@article{huang2025selfforcing,
  title={Self Forcing: Bridging the Train-Test Gap in Autoregressive Video Diffusion},
  author={Huang, Xun and Li, Zhengqi and He, Guande and Zhou, Mingyuan and Shechtman, Eli},
  journal={arXiv preprint arXiv:2506.08009},
  year={2025}
}

\appendix

\newpage
\appendix
\section{Additional Theory}
\label{app: additional theory}
\subsection{teacher score discrepancy}
A good regularizer must not bias the solution: if the student is already perfect, the regularizer should contribute zero gradient. We show this holds for $\Delta_{p_{\mathrm{real}},\bm{x},G_\theta(\bm{z})}$ in expectation.

\textbf{Proposition.} If the student matches the data distribution exactly ($p_{\mathrm{fake}} = p_{\mathrm{real}}$), then $\E_{\bm{x},\bm{z}}[\Delta_{p_{\mathrm{real}},\bm{x},G_\theta(\bm{z})}] = 0$.

\textbf{Proof.} When the student is perfect, $\bm{x} = G_\theta(\bm{z})$ is a sample from $p_{\mathrm{real}}$, and $\hat{\bm{x}}$ is also a sample from $p_{\mathrm{real}}$. They are i.i.d.\ and therefore \emph{exchangeable}. With shared $\bm{\epsilon}$,
\begin{equation}
\bm{x}^{\mathrm{fake}}_{t} = \alpha_t \bm{x} + \sigma_t \bm{\epsilon} \quad \overset{d}{=} \quad \alpha_t \hat{\bm{x}} + \sigma_t \bm{\epsilon} = \bm{x}^{\mathrm{real}}_{t}.
\end{equation}
Both have the same distribution $p_{\text{real}, t}$, so

\begin{equation}
\E_{\bm{x},\bm{z}}[\Delta_{p_{\mathrm{real}},\bm{x},G_\theta(\bm{z})}] = \E_{\bm{x},\bm{z}}[\nabla_{\bm{x}} \log p_{\mathrm{real}}(\bm{x}^{\mathrm{real}}_{t})] - \E_{\bm{x},\bm{z}}[\nabla_{\bm{x}} \log p_{\mathrm{real}}(\bm{x}^{\mathrm{fake}}_{t}))] = 0.
\end{equation}

\textbf{$\Delta_{p_{\mathrm{real}},\bm{x},G_\theta(\bm{z})}$ is a zero-mean stochastic perturbation at the optimum.} DFD does not bias convergence: it adds a regularizer that vanishes in expectation precisely when the student is correct. \hfill$\square$

\subsection{Bounded MSE between $g_{\mathrm{DFD}}(\theta)$ and $\nabla_\theta D_{\mathrm{KL}}(p_{\mathrm{fake}} \| p_{\mathrm{real}})$}

\subsubsection*{Setup}

In score-distillation methods, the teacher score is never evaluated on clean inputs. Following standard practice, samples are perturbed by the diffusion forward process to a noise level $t \in (0, T]$ before being passed to the teacher. Let

\begin{equation}
\bm{x}_t := \alpha_t \bm{x} + \sigma_t \bm{\epsilon}, \bm{x} \sim p_{\mathrm{real}},\qquad \tilde{\bm{x}}_t := \alpha_t G_\theta(\bm{z}, c) + \sigma_t \bm{\epsilon}, \qquad \bm{\epsilon} \sim \mathcal{N}(\bm{0}, \bm{I}),
\end{equation}

where $(\alpha_t, \sigma_t)$ are the diffusion schedule coefficients. The relevant scores are those of the noised marginals,

\begin{equation}
s_{\mathrm{real}}(\bm{x}_t \mid c, t) := \nabla_{\bm{x}_t}\log p_{\mathrm{real}, t}(\bm{x}_t \mid c), \qquad
s_{\mathrm{fake}}(\tilde{\bm{x}}_t \mid c, t) := \nabla_{\tilde{\bm{x}}_t}\log p_{\mathrm{fake}, t}(\tilde{\bm{x}}_t \mid c),
\end{equation}

and the teacher is a neural network $s_\phi(\cdot, c, t)$ approximating $s_{\mathrm{real}}(\cdot \mid c, t)$.

The DFD and KL gradient at noise level $t$ are:
\begin{align}
g_{\mathrm{DFD}}(\theta) &= \E_{\bm{x}, \bm{z}, \bm{\epsilon}}\!\left[\bigl(s_{\mathrm{fake}}(\tilde{\bm{x}}_t \mid c, t) - s_{\mathrm{real}}(\bm{x}_t \mid c, t)\bigr)\nabla_\theta G_\theta(\bm{z},c)\right], \label{eq:rfsm}\\
\nabla_\theta D_{\mathrm{KL}} &= \E_{\bm{z}, \bm{\epsilon}}\!\left[\bigl(s_{\mathrm{fake}}(\tilde{\bm{x}}_t \mid c, t) - s_{\mathrm{real}}(\tilde{\bm{x}}_t \mid c, t)\bigr)\nabla_\theta G_\theta(\bm{z},c)\right]. \label{eq:kl}
\end{align}
(For brevity we have absorbed any $\alpha_t$ chain-rule factors into $\nabla_\theta G_\theta$ without changing the structure of the argument.)

\subsubsection*{Condition}

\begin{condition}[Bounded conditional matching error]
\label{ass:matching}
There exists $\delta(c) \geq 0$ such that
\begin{equation}
\E\!\left[\norm{\bm{x} - G_\theta(\bm{z},c)}^{2} \;\big|\; c\right] \leq \delta(c)^{2}.
\end{equation}
\end{condition}

\begin{condition}[Lipschitz teacher score on the noised diffusion path]
\label{cond:lipschitz}
At each noise level $t > 0$, the teacher score network $s_{\mathrm{real}}(\cdot \mid c, t)$ is $L(c, t)$-Lipschitz:

\begin{equation}
\norm{s_{\mathrm{real}}(\bm{u} \mid c, t) - s_{\mathrm{real}}(\bm{v} \mid c, t)} \leq L(c, t)\,\norm{\bm{u} - \bm{v}}, \qquad \forall \bm{u}, \bm{v} \in \R^{d}.
\end{equation}

This is justified for two complementary reasons: (i) the noised marginal $p_{\mathrm{real}, t}$ is the convolution of $p_{\mathrm{real}}$ with a Gaussian of variance $\sigma_t^2$, whose score has Hessian operator norm bounded by $\mathcal{O}(1/\sigma_t^2)$; (ii) the score is parameterized by a neural network with Lipschitz architectural primitives, so it inherits a finite Lipschitz constant on any bounded domain.
\end{condition}

\begin{condition}[Bounded generator Jacobian]
\label{ass:jac}
$\norm{\nabla_\theta G_\theta(\bm{z},c)}_{\mathrm{op}} \leq B$ almost surely.
\end{condition}

\subsubsection*{Main result}

\begin{proposition}
Under Condition~\ref{ass:matching}--\ref{ass:jac}, for every noise level $t > 0$,

\begin{equation}    
\bigl\|g_{\mathrm{DFD}}(\theta) - \nabla_\theta D_{\mathrm{KL}}(p_{\mathrm{fake}} \| p_{\mathrm{real}})\bigr\|_{2}^{2}
\;\leq\;
B^{2}\,\alpha_t^{2}\,L(c, t)^{2}\,\delta(c)^{2}.
\end{equation}

\end{proposition}

\begin{proof}
\textbf{Step 1: Cancellation of the fake score.}
The fake-score term in \eqref{eq:rfsm} depends only on $(\bm{z}, \bm{\epsilon}, c)$, not on $\bm{x}$, so $\E_{\bm{x},\bm{z},\bm{\epsilon}\mid c}[\,\cdot\,] = \E_{\bm{z},\bm{\epsilon}\mid c}[\,\cdot\,]$ for that term. It cancels with the corresponding term in \eqref{eq:kl}, giving

\begin{equation}
g_{\mathrm{DFD}}(\theta) - \nabla_\theta D_{\mathrm{KL}}
= \E_{\bm{x}, \bm{z}, \bm{\epsilon} \mid c}\!\left[\bigl(s_{\mathrm{real}}(\tilde{\bm{x}}_t \mid c, t) - s_{\mathrm{real}}(\bm{x}_t \mid c, t)\bigr)\nabla_\theta G_\theta(\bm{z}, c)\right].
\end{equation}

\textbf{Step 2: Jensen's inequality.}
Since $\norm{\cdot}_2^2$ is convex,
\begin{equation}
\bigl\|g_{\mathrm{DFD}}(\theta) - \nabla_\theta D_{\mathrm{KL}}\bigr\|_{2}^{2}
\;\leq\;
\E_{\bm{x}, \bm{z}, \bm{\epsilon} \mid c}\!\left[\norm{\bigl(s_{\mathrm{real}}(\tilde{\bm{x}}_t \mid c, t) - s_{\mathrm{real}}(\bm{x}_t \mid c, t)\bigr)\nabla_\theta G_\theta(\bm{z},c)}_{2}^{2}\right].
\end{equation}

\textbf{Step 3: Bounding the integrand.}
Using sub-multiplicativity, Condition~\ref{cond:lipschitz}, and Condition~\ref{ass:jac},
\begin{equation}
\begin{aligned}
\norm{\bigl(s_{\mathrm{real}}(\tilde{\bm{x}}_t \mid c, t) - s_{\mathrm{real}}(\bm{x}_t \mid c, t)\bigr)\nabla_\theta G_\theta(\bm{z},c)}_{2}^{2}
&\leq B^{2}\,\norm{s_{\mathrm{real}}(\tilde{\bm{x}}_t \mid c, t) - s_{\mathrm{real}}(\bm{x}_t \mid c, t)}^{2} \\
&\leq B^{2}\,L(c, t)^{2}\,\norm{\tilde{\bm{x}}_t - \bm{x}_t}^{2}.
\end{aligned}
\end{equation}
Since $\tilde{\bm{x}}_t - \bm{x}_t = \alpha_t\bigl(G_\theta(\bm{z},c) - \bm{x}\bigr)$ (the noise term $\sigma_t \bm{\epsilon}$ is shared and cancels),
\begin{equation}
\norm{\tilde{\bm{x}}_t - \bm{x}_t}^{2} = \alpha_t^{2}\,\norm{G_\theta(\bm{z},c) - \bm{x}}^{2}.
\end{equation}

\textbf{Step 4: Apply the matching bound.}
Combining the previous steps,
\begin{equation}
\begin{aligned}
\bigl\|g_{\mathrm{DFD}}(\theta) - \nabla_\theta D_{\mathrm{KL}}\bigr\|_{2}^{2}
&\leq B^{2}\,\alpha_t^{2}\,L(c, t)^{2}\,\E\!\left[\norm{\bm{x} - G_\theta(\bm{z},c)}^{2}\,\big|\,c\right] \\
&\leq B^{2}\,\alpha_t^{2}\,L(c, t)^{2}\,\delta(c)^{2}. \qedhere
\end{aligned}
\end{equation}
\textbf{which is equavalent to:}
\begin{equation}
\begin{aligned}
\E\!\bigl[\bigl\|\Delta_{p_{\mathrm{real}},\bm{x},G_\theta(\bm{z})}\bigr\|_{2}^{2} \,\big|\, c\bigr]
    \;\lesssim\; \alpha_t^{2}\,L(c,t)^{2}\,\delta(c)^{2}
\end{aligned}
\end{equation}
\end{proof}





\section{Experiment Details}
\label{app: experiment detail}


\subsection{Data Curation}
\label{app sub: dataset curation}
We mainly curated two datasets for our experiments:
\begin{itemize}
    \item \emph{animation dataset}: video clips with an animation or cartoon visual style.
    \item \emph{mix-style dataset}: video clips spanning all visual styles.
\end{itemize}
Both are derived from the ViPE dataset~\cite{huang2025vipe}, starting from $\sim$966k captioned source clips. Our pipeline has five stages: (1) tri-level captioning, (2) CLIP visual feature extraction, (3) $K$-means clustering, (4) per-cluster grid visualization with human-in-the-loop selection, and (5) WebDataset packaging at three caption granularities. The full pipeline code is released alongside the paper.

\subsubsection{Tri-Level Captioning.}
For each source video we generate three captions of different lengths in a single VLM forward pass, using \textbf{Qwen3-VL-4B-Instruct} served via vLLM with continuous batching and PagedAttention. A single system prompt instructs the model to emit
\begin{itemize}
    \item \texttt{[LONG]}: a $250$--$300$ word caption covering subject appearance, actions, expressions, environment, camera work, and visual style;
    \item \texttt{[MEDIUM]}: a $150$--$180$ word caption that retains the key subject, actions, setting, and style;
    \item \texttt{[SHORT]}: a single sentence capturing the essence of the clip.
\end{itemize}
Captions are decoded with sampling parameters $T=0.7$, $\text{top-}p=0.8$, and \texttt{max\_new\_tokens}$=512$. Producing all three lengths in one forward pass (rather than three independent passes) reduces VLM cost roughly $3\times$ since the visual tokens are encoded only once. The three captions are parsed from the tagged output via regular expressions; if any tag is missing we fall back deterministically (truncate the long caption for medium, take the first sentence for short). Each caption is also encoded with the UMT5-XXL text encoder (Wan2.1-T2V-1.3B checkpoint) and the resulting prompt embeddings are stored alongside the raw text.

\subsubsection{Visual Feature Extraction.}
For diversity sampling we want a feature space that reflects visual content rather than caption phrasing. We therefore extract CLIP ViT-B/32 image feature. For each video we decode $3$ uniformly-spaced keyframes with PyAV, apply the standard CLIP preprocessing (\texttt{Resize}(224), \texttt{CenterCrop}(224), \texttt{ToTensor}, CLIP normalization), encode each keyframe with \texttt{encode\_image}, mean-pool the three frame embeddings, and $\ell_2$-normalize, yielding one 512-d vector per video. 

\subsubsection{Clustering.}
We cluster the $\sim$966k feature vectors with \textbf{FAISS $K$-means} ($K{=}1000$, $50$ iterations). Because all vectors are $\ell_2$-normalized, squared Euclidean distance to the centroid is monotonically related to cosine similarity, so we use it directly for ranking inside each cluster. The output is a per-video integer cluster label and a $(K{,}512)$ centroid matrix.

\subsubsection{Per-Cluster Visualization and Human-in-the-Loop Selection.}
For each of the $K{=}1000$ clusters we (i) sort all members by Euclidean distance from the centroid, (ii) pick $40$ candidates at evenly-spaced percentiles of that distance (so row $0$ is the centroid-nearest video and row $39$ is the farthest), and (iii) decode $5$ evenly-spaced frames per candidate. 
A human reviewer then opens each cluster grid and writes the row indices of $20$ visually-representative-but-distinct videos into the \texttt{picks} column, e.g.\ \texttt{"0,2,5,7,9,12,...,38,39"}. The even-percentile sampling ensures the $40$ candidates already span the full intra-cluster spread, so the reviewer is choosing among visually distinct options rather than near-duplicates. Clusters whose grid contains no useful content (e.g., uniformly black frames, decode-failed clusters) can be left empty and are dropped from the final dataset; this trades a small number of slots for higher quality. After review we obtain $30{,}000$ selected videos in cluster-id order.


The animation dataset is constructed by the same filtering pipeline, restricted to clusters whose grid visualizations are dominated by animation/cartoon style during the human review step.

\subsection{Baseline Details}
\label{app sub: baseline detail}
We show the implementation details of our baselines in Table~\ref{tab:configs}.
\begin{table*}[htbp]
\centering
\caption{Training configurations.}
\label{tab:configs}
\resizebox{\textwidth}{!}{ 
\begin{tabular}{lccc|ccc}
\toprule
\multirow{2}{*}{\textbf{Models}} & \multicolumn{3}{c|}{\textbf{Cosmos Predict2 I2V}} & \multicolumn{3}{c}{\textbf{Wan2.1 T2V}} \\
\cmidrule(lr){2-4} \cmidrule(l){5-7}
& \textbf{DMD2} & \textbf{DP-DMD} & \textbf{DFD} & \textbf{DMD2} & \textbf{DP-DMD} &\textbf{DFD} \\
\midrule
Num. of Frame & 81 & 81 & 81 & 81 & 81 & 81 \\
Batch Size & 16 & 16 & 16 & 16 & 16& 16 \\
Resolution &  480p& 480p& 480p&480p&480p&480p \\
Learning Rate(discriminator) & 1e-05 & 1e-05 & N/A & 1e-05 & 1e-05 & N/A \\
Learning Rate (student) & 1e-05 & 1e-05 & 1e-05 & 1e-05 & 1e-05 & 1e-05\\
Learning Rate (fake score)& 1e-05 & 1e-05 & 1e-05 & 1e-05 & 1e-05 & 1e-05\\
CFG Scale & 3.0& 3.0 & 3.0 & 5.0 & 5.0& 5.0 \\
Student Update Frequency & 5 & 5 & 5 & 5 & 5& 5 \\
Diversity anchor step  & N/A & 1 & N/A & N/A & 5 &N/A\\
Diversity weights    &     N/A &   0.05  &  N/A   &   N/A  & $0.05$ &N/A \\

Total Iterations & 30k & 30k & 300 & 30k & 30k & 100\\
Pretrained Model Iterations & 0 & 0 & 25k(from DMD2) & 0 & 0 & 25k(from DMD2)\\

\bottomrule
\end{tabular}
}
\end{table*}
For the text-to-video tasks, all the baseline and our method are trained on 8 A100 GPUs. For the image-to-video tasks, all the baseline and our method are trained on 16 A100 GPUs.
\subsection{Text-to-Video}
\label{app sub: t2v}
\paragraph{Metric Details}
We show the metric details here. For each text prompt we sample $L=8$ videos from distinct initial noise seeds, and report three families of metrics: a visual diversity metric adapted from DP-DMD, a set of VBench~\cite{huang2024vbench} quality dimensions, and camera-pose statistics extracted with ViPE~\cite{huang2025vipe}.

\subsubsection{Visual Diversity Metric.} We adopt the diversity metric from DP-DMD \cite{wu2026diversity}. We assess sample diversity using DINOv2-ViT-Large (DINO)~\cite{simeoni2025dinov3} and CLIP-ViT-Large (CLIP)~\cite{radford2021learning} by computing the cosine similarity between extracted image feature representations in a pairwise manner:
\begin{equation}
\mathrm{Diversity}
=1-\frac{2}{L(L-1)}\sum_{i,\,j}
\cos\Big(\bm{x}^{(i)}_\theta,\bm{x}^{(j)}_\theta\Big),
\label{eq:div_eval}
\end{equation}
where $L$ denotes the number of distinct initial noise samples per text prompt\footnote{For each text prompt, every pair of generated samples is compared once, yielding $\binom{L}{2}$ total comparisons.}, and we set $L=8$ in our experiments. To adapt this image-level metric to videos, we sample $16$ evenly-spaced frames from each generated clip, extract per-frame CLIP and DINO embeddings, and mean-pool them into a single feature vector $\bm{x}^{(i)}_\theta$ per video before applying Eq.~\ref{eq:div_eval}. The reported diversity is the average over all text prompts.

\subsubsection{VBench Quality Metrics.} 
For t2v tasks, We evaluate generation quality with VBench~\cite{huang2024vbench} which scores videos directly without requiring access to the original training prompts. We report the seven prompt-free dimensions:
\begin{itemize}
    \item \emph{subject consistency}: measures whether the main subject's appearance (identity, shape, texture) remains stable across frames, computed via DINO feature similarity between frames.
    \item \emph{background consistency}: measures temporal stability of the background scene across frames using CLIP feature similarity, capturing flicker or drift in the surrounding environment.
    \item \emph{temporal flickering}: penalizes high-frequency, low-level fluctuations between adjacent frames in static or near-static regions, reflecting visible flicker artifacts.
    \item \emph{motion smoothness}: assesses whether motion follows the priors of a video frame interpolation model, with smoother (more physically plausible) motion scoring higher.
    \item \emph{dynamic degree}: estimates the amount of motion in the video using optical flow magnitude, rewarding generations that are not nearly-static.
    \item \emph{aesthetic quality}: predicts the human-perceived aesthetic score of sampled frames using the LAION aesthetic predictor.
    \item \emph{imaging quality}: predicts low-level image quality (sharpness, noise, distortion) using the MUSIQ image-quality predictor.
\end{itemize}
For i2v tasks, we evaluate the following metrics provided in VBench i2v part:

\begin{itemize}
    \item \textit{subject consistency:} Measures the model's ability to preserve the identity, appearance, and structural integrity of the primary subject across all frames of the generated video.
    \item \textit{Background Consistency: }Evaluates the temporal stability and visual coherence of the background environment, penalizing unnatural shifts or morphing over time.
    \item \textit{Aesthetic Quality: }Assesses the overall visual appeal, lighting, composition, and artistic fidelity of the individual generated frames.
    \item \textit{Temporal Flickering: }Quantifies the presence of high-frequency visual artifacts or unnatural, rapid changes in texture and lighting between adjacent frames.
    \item \textit{Motion Smoothness:} Evaluates the naturalness, fluidity, and physical plausibility of the movements occurring within the video sequence.
    \item \textit{Image-to-Video Subject:} Specifically for Image-to-Video (I2V) generation tasks, this metric calculates how accurately the subject in the generated video aligns with the reference subject provided in the conditioning input image.
    \item \textit{Image-to-Video Background:} Evaluates how well the generated video maintains and reflects the background elements present in the initial conditioning image.
\end{itemize}
\subsubsection{Camera-Pose Diversity.} To quantify how the camera moves across generated videos, we run ViPE~\cite{huang2025vipe} on every clip to recover an $81$-frame sequence of $4{\times}4$ camera-to-world transformation matrices $\{T_t\}_{t=0}^{80}$, with translation $\bm{p}_t = T_t[{:}3,3]$ and rotation $R_t = T_t[{:}3,{:}3]$. We then derive two groups of metrics from these poses.

\noindent\emph{Per-video motion magnitude} (computed independently for each video):
\begin{itemize}
    \item \emph{translation path length}: $\sum_{t=0}^{N-2}\,\lVert \bm{p}_{t+1}-\bm{p}_t\rVert$, the total distance the camera travels.
    \item \emph{max translation}: $\max_{t}\,\lVert \bm{p}_t-\bm{p}_0\rVert$, the farthest the camera reaches from its starting position.
    \item \emph{rotation total} (degrees): $\sum_{t=0}^{N-2}\,\angle\!\big(R_{t+1} R_t^{\top}\big)$, the accumulated frame-to-frame rotation magnitude.
    \item \emph{max rotation} (degrees): $\max_{t}\,\angle\!\big(R_t R_0^{\top}\big)$, the largest deviation from the initial orientation.
\end{itemize}

\noindent\emph{Cross-seed trajectory diversity} (computed per prompt across the $L=8$ seeds, then averaged over prompts). Let $\bm{e}^{(i)}=\bm{p}^{(i)}_{N-1}$ be the endpoint of the $i$-th seed and $\{\bm{p}^{(i)}_t\}$ its full trajectory:
\begin{itemize}
    \item \emph{endpoint spread}: $\big\lVert \mathrm{std}_i\bm{e}^{(i)} \big\rVert_2$, the $\ell_2$ norm of the per-axis standard deviation of the endpoints.
    \item \emph{mean pairwise endpoint distance}: $\frac{2}{L(L-1)}\sum_{i<j}\lVert \bm{e}^{(i)}-\bm{e}^{(j)}\rVert$, averaged over all $\binom{L}{2}=28$ seed pairs.
    \item \emph{mean pairwise trajectory distance}: $\frac{2}{L(L-1)}\sum_{i<j}\,\frac{1}{N}\sum_{t}\lVert \bm{p}^{(i)}_t-\bm{p}^{(j)}_t\rVert$, the average per-frame distance between trajectories, capturing path-level (not just endpoint) differences.
\end{itemize}

\section{Additional Results}
\label{app: additional results}
\subsection{Additional Results for Text-to-Video Generation}
\label{app sub: additional results for i2v}
We show more visualization results for the text-to-video generation. And \textbf{We refer the readers to the supplementary material for full video comparisons}


\subsubsection{Diversity Visualization Results}
\label{app sub: diversity visualization results}
We further illustrate the diversity advantage of our method by generating videos with $8$ different random seeds under the same text prompt. Fig. ~\ref{fig:appendix_diversity} shows the middle frame of each generated video for DMD2, DP-DMD, and ours. Each row corresponds to a different pair of seeds (rows 1--4 cover seeds $1$--$8$). Within each method, the two columns show two different seeds; within each row, the same two seeds are shown across the three methods. Compared to DMD2 and DP-DMD, our method produces clearly more diverse outputs across seeds, with larger variation in subject pose, layout, and scene composition.

\begin{figure}[htbp]
    \centering
    \setlength{\tabcolsep}{2pt}
    \renewcommand{\arraystretch}{1.1}
    \begin{tabular}{@{}cccccc@{}}
        \multicolumn{2}{c}{\small DMD2} &
        \multicolumn{2}{c}{\small DP-DMD} &
        \multicolumn{2}{c}{\small Ours} \\
        \includegraphics[width=0.155\linewidth]{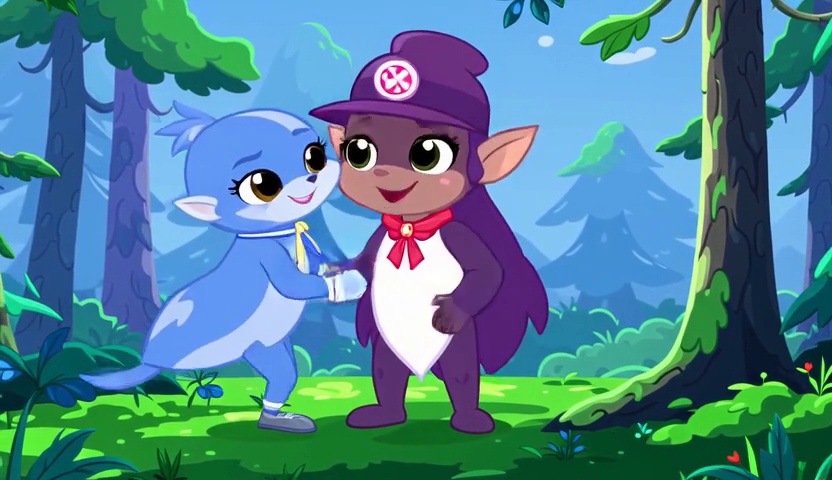} &
        \includegraphics[width=0.155\linewidth]{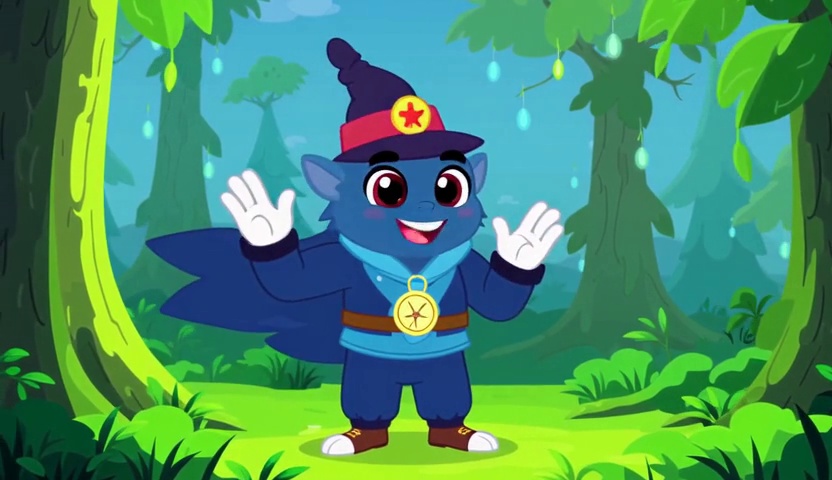} &
        \includegraphics[width=0.155\linewidth]{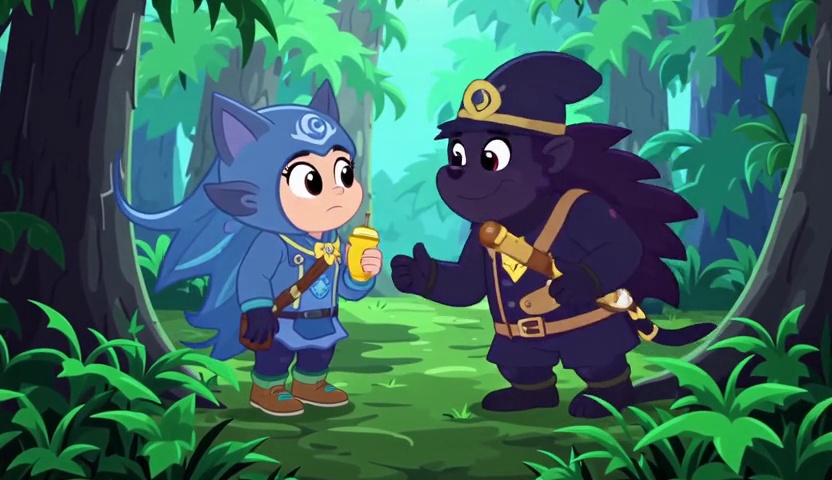} &
        \includegraphics[width=0.155\linewidth]{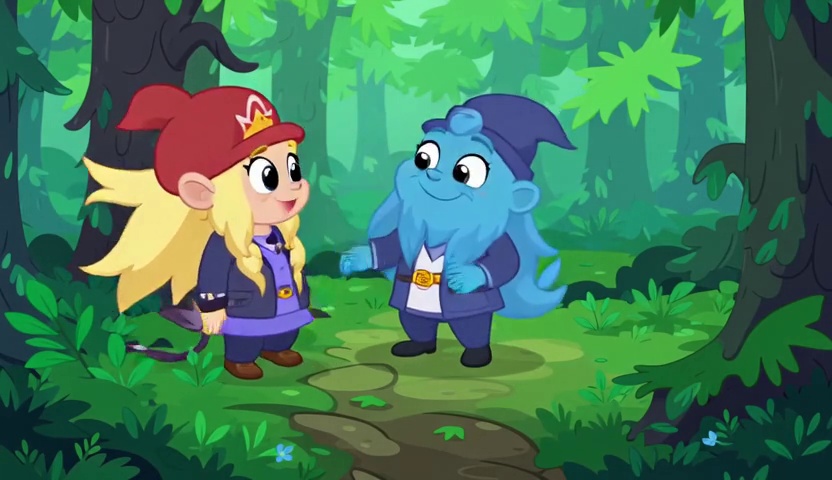} &
        \includegraphics[width=0.155\linewidth]{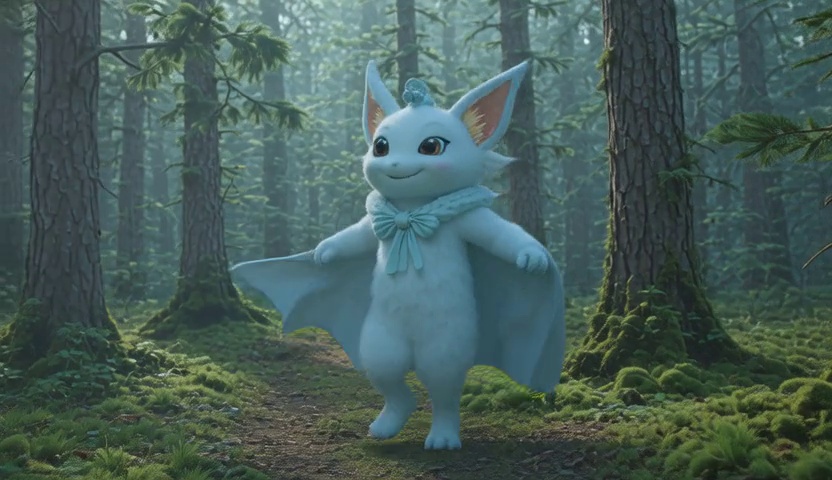} &
        \includegraphics[width=0.155\linewidth]{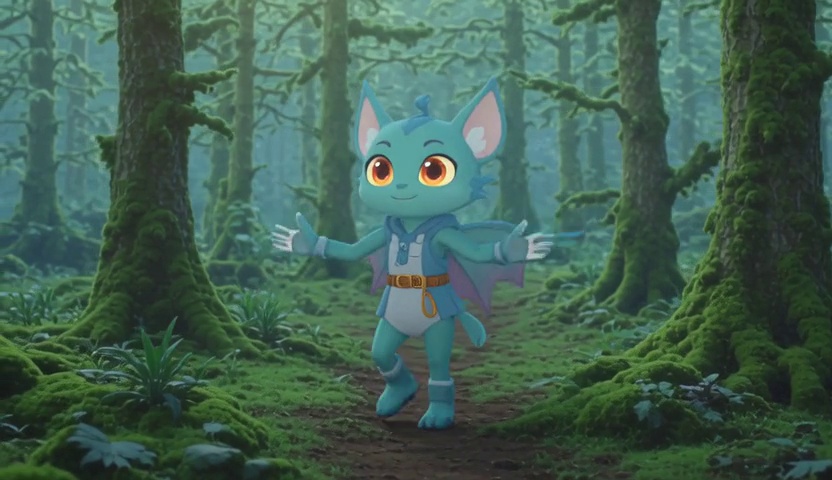} \\
        \includegraphics[width=0.155\linewidth]{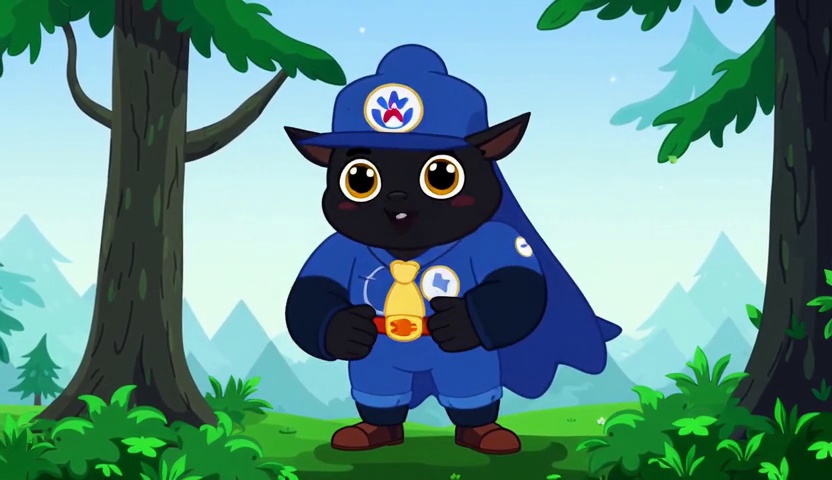} &
        \includegraphics[width=0.155\linewidth]{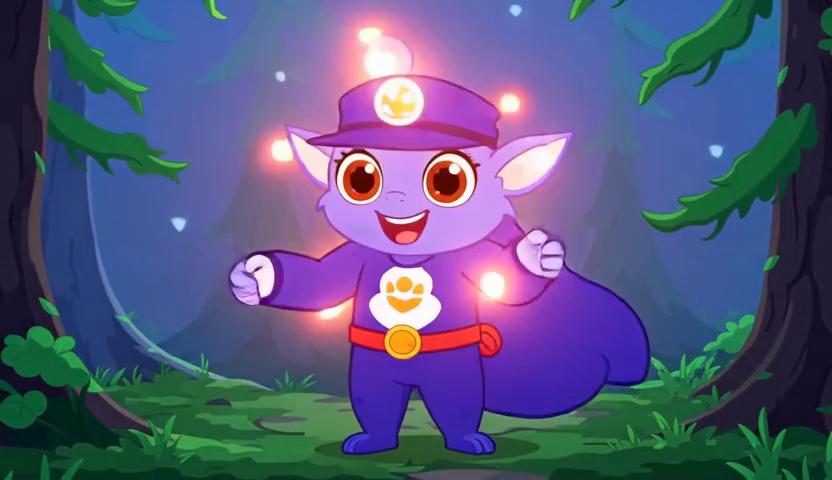} &
        \includegraphics[width=0.155\linewidth]{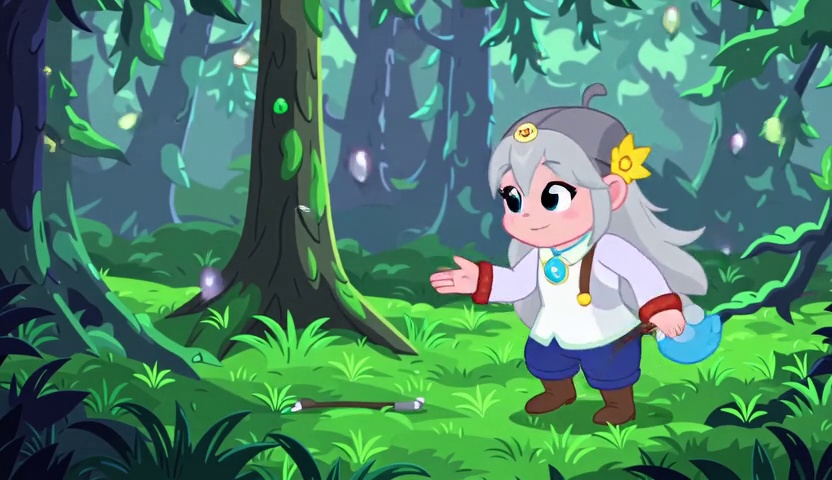} &
        \includegraphics[width=0.155\linewidth]{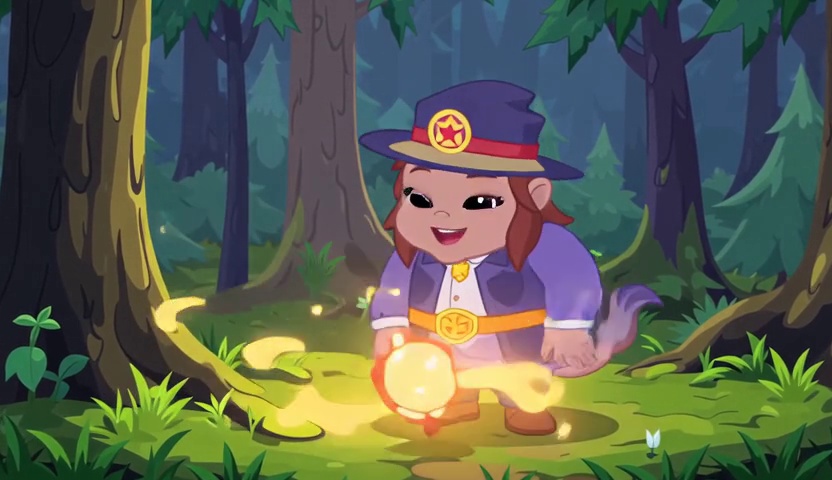} &
        \includegraphics[width=0.155\linewidth]{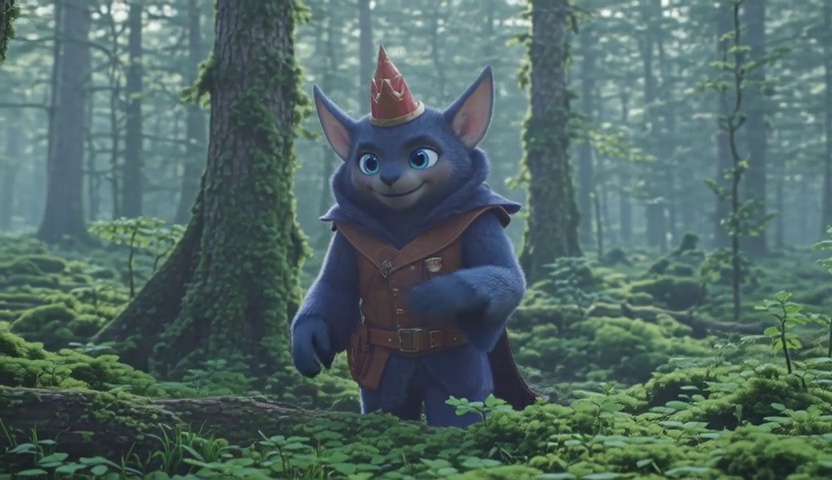} &
        \includegraphics[width=0.155\linewidth]{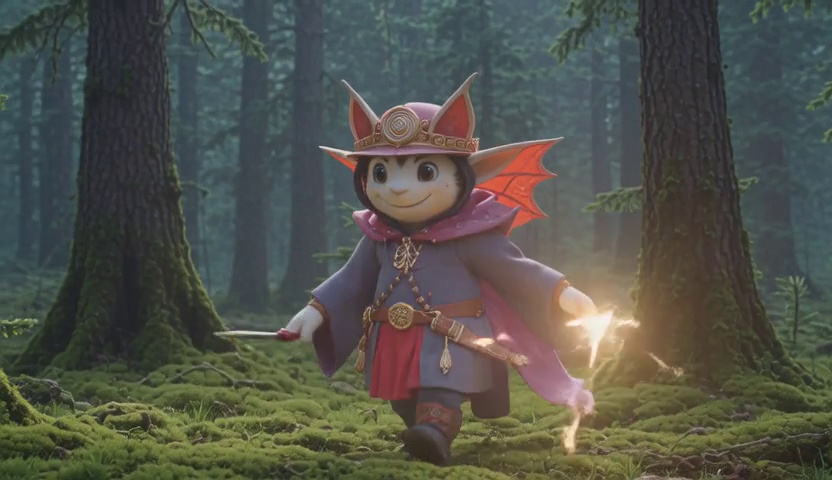} \\
        \includegraphics[width=0.155\linewidth]{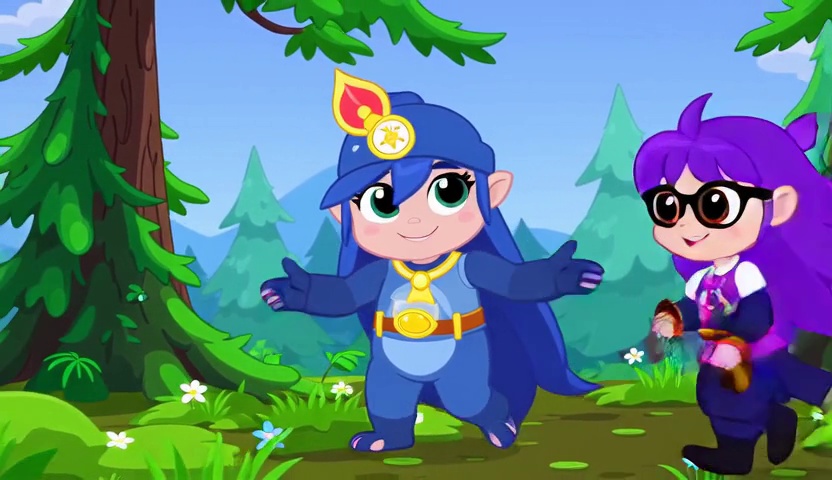} &
        \includegraphics[width=0.155\linewidth]{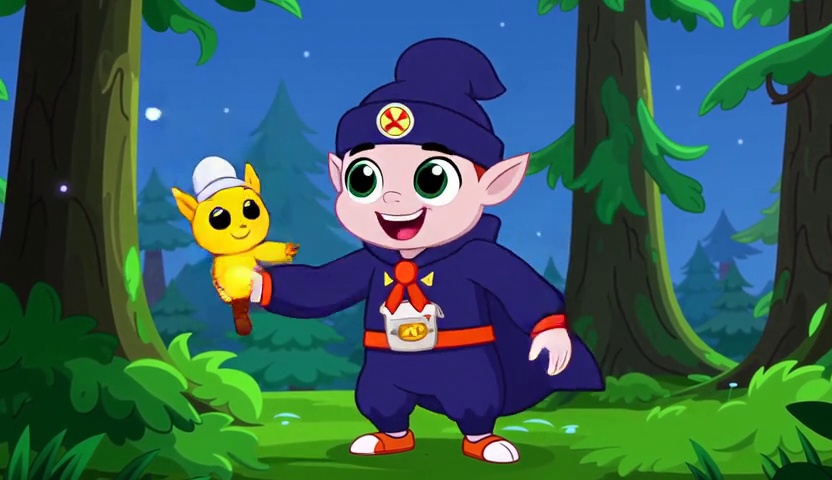} &
        \includegraphics[width=0.155\linewidth]{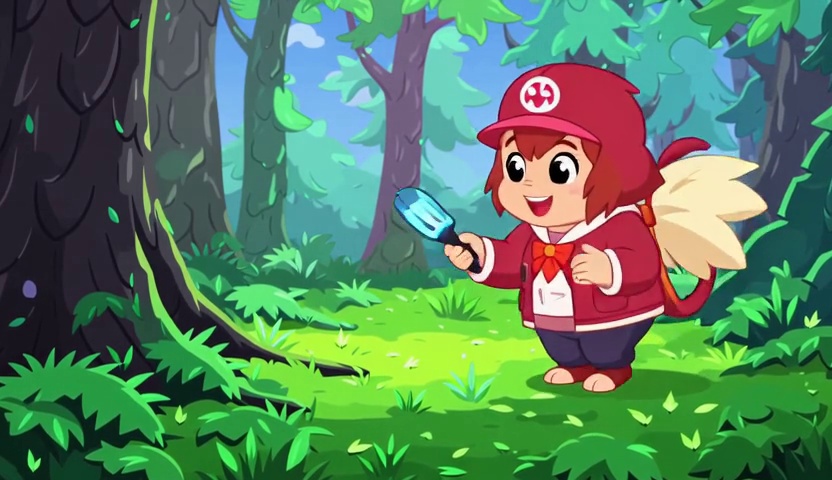} &
        \includegraphics[width=0.155\linewidth]{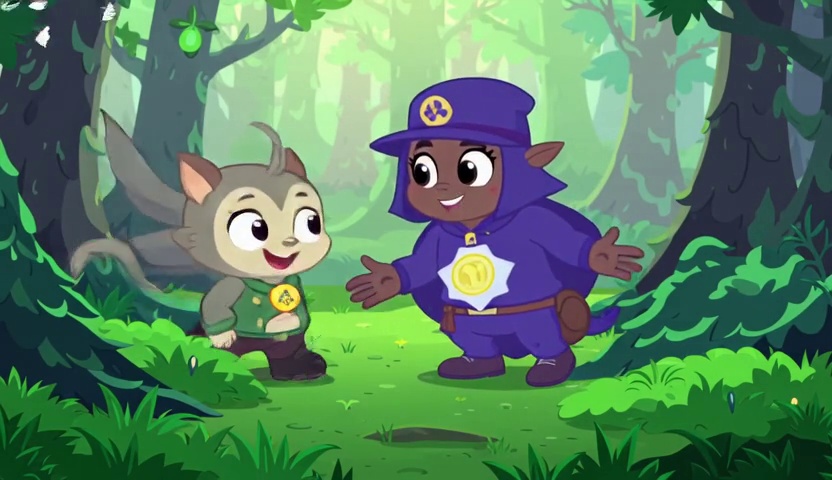} &
        \includegraphics[width=0.155\linewidth]{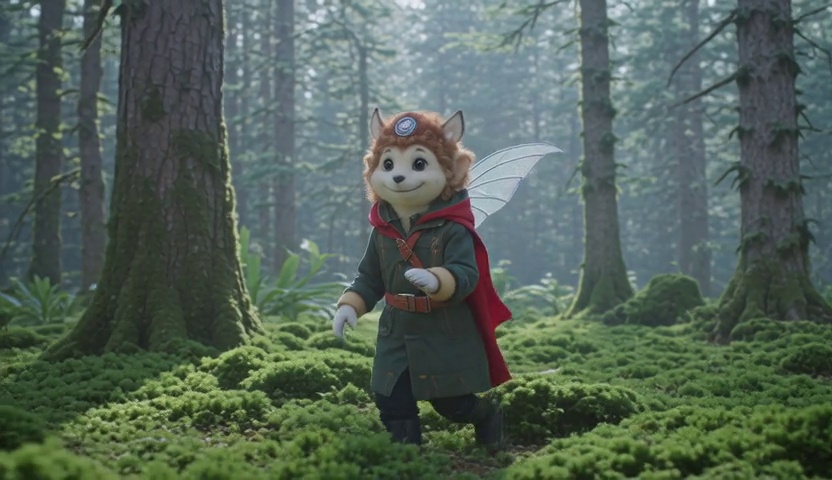} &
        \includegraphics[width=0.155\linewidth]{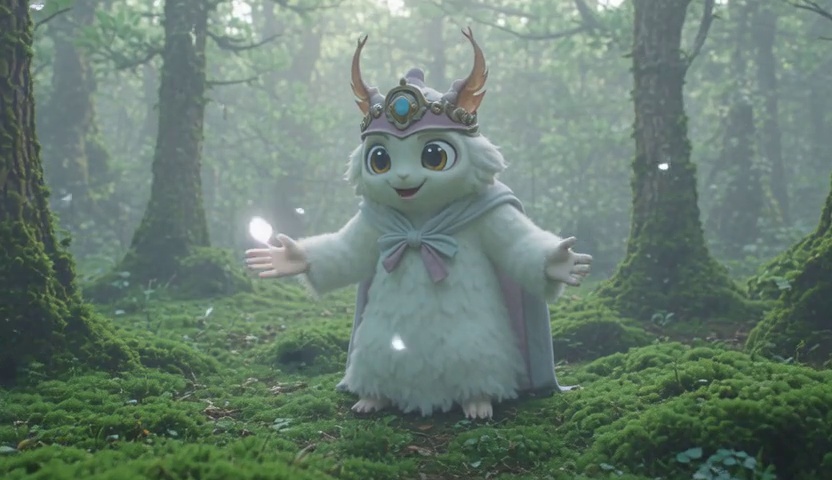} \\
        \includegraphics[width=0.155\linewidth]{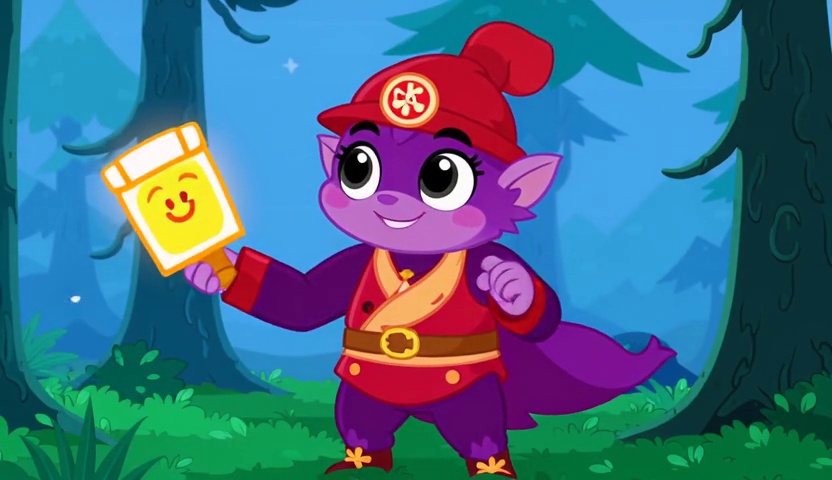} &
        \includegraphics[width=0.155\linewidth]{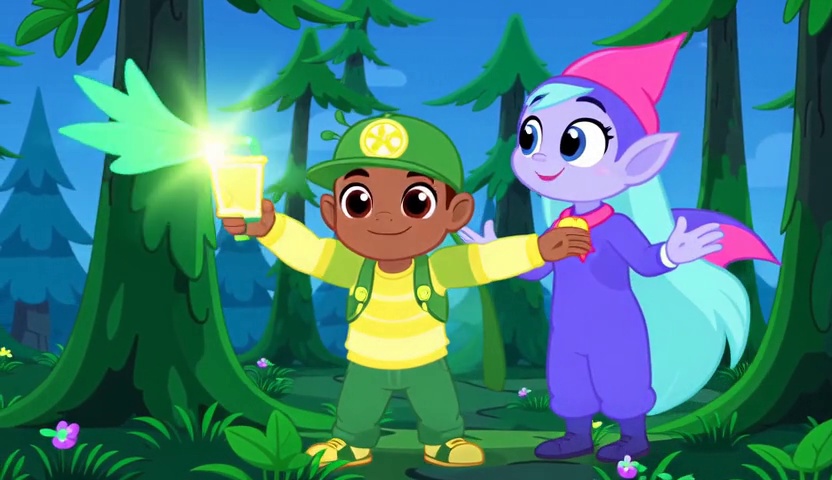} &
        \includegraphics[width=0.155\linewidth]{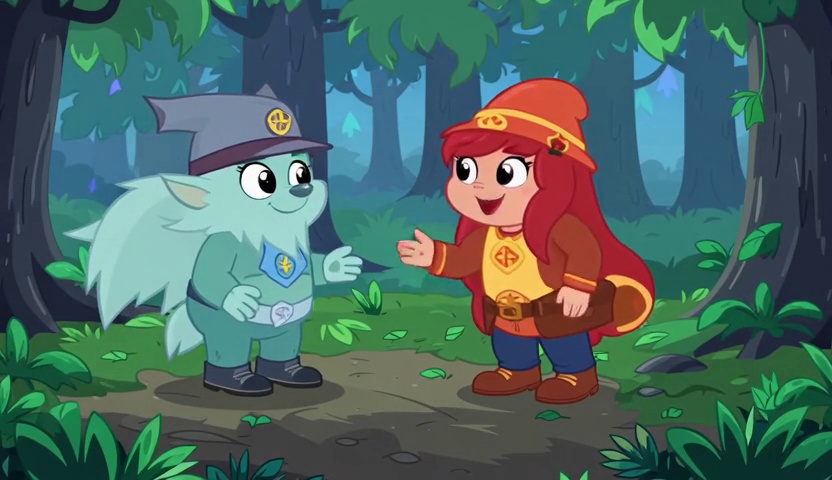} &
        \includegraphics[width=0.155\linewidth]{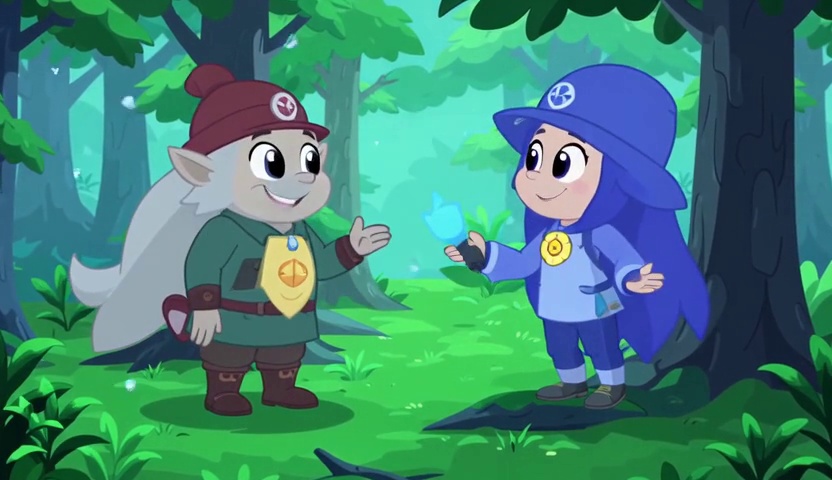} &
        \includegraphics[width=0.155\linewidth]{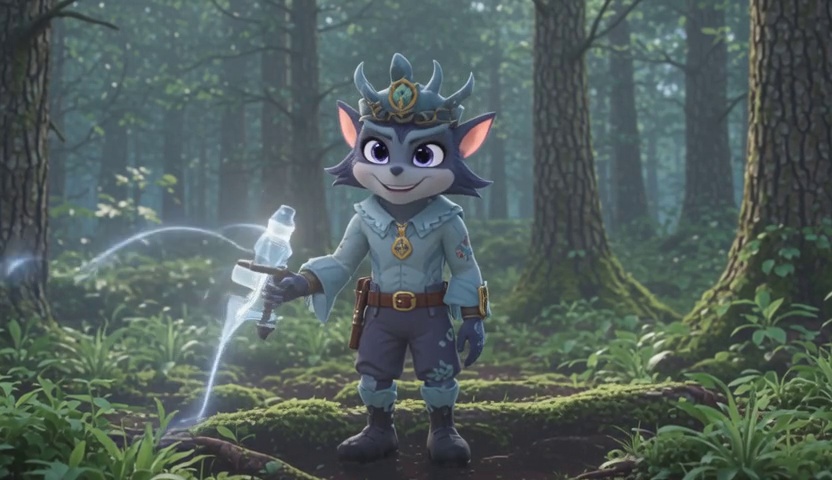} &
        \includegraphics[width=0.155\linewidth]{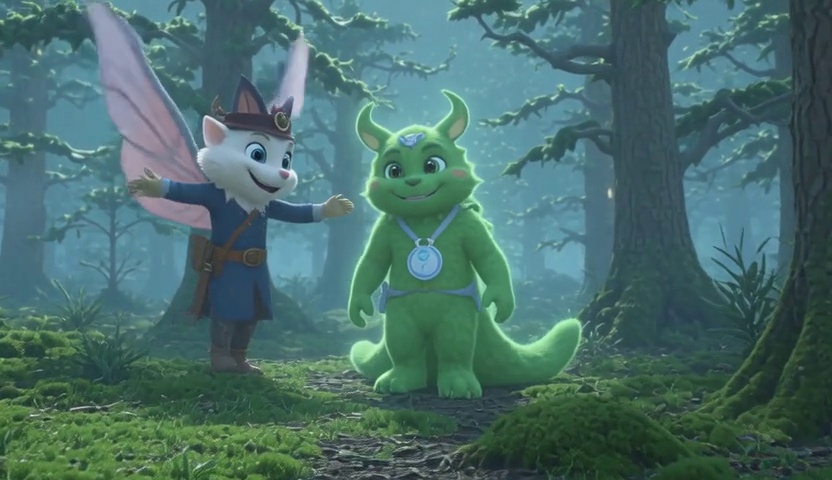} \\
    \end{tabular}
    \caption{\textbf{Diversity visualization across 8 random seeds for the same prompt}. Columns 1--2: DMD2; columns 3--4: DP-DMD; columns 5--6: Ours. Each row shows the middle frame of two videos generated with different seeds. Our method produces visibly more diverse outputs across seeds.}
    \label{fig:appendix_diversity}
\end{figure}
\subsection{Main Results for Auto}
\subsection{Additional Results for Image-to-Video Generation}
\begin{table*}[h]
\centering
\caption{{Quantitative comparison on image-to-video generation}. We evaluate the metrics from VBench on our curated ViPE test set. For all metrics, higher is better. The best result in each column is \textbf{bolded}. Our method consistently outperforms DMD2 and DP-DMD, and is even better than the teacher model.}
\setlength{\tabcolsep}{4pt}
\renewcommand{\arraystretch}{1.2}
\resizebox{\textwidth}{!}{%
\begin{tabular}{l|ccccccc|c}
\toprule
\textbf{Method} &
\shortstack{\textbf{Subject}\\\textbf{Consistency}} &
\shortstack{\textbf{Background}\\\textbf{Consistency}} &
\shortstack{\textbf{Aesthetic}\\\textbf{Quality}} &
\shortstack{\textbf{Temporal}\\\textbf{Flickering}} &
\shortstack{\textbf{Motion}\\\textbf{Smoothness}} &
\shortstack{\textbf{Image-to-Video}\\\textbf{Subject}} &
\shortstack{\textbf{Image-to-Video}\\\textbf{Background}} &
\textbf{Average} \\
\midrule
Teacher   & 0.9340 & 0.9468 & 0.6026 & 0.9694 & 0.9886 & 0.9781 & 0.9860 & 0.9151 \\
\midrule
DMD2    & 0.9340 & 0.9473 &\textbf{ 0.6095} & 0.9722 & 0.9892 & 0.9805 & 0.9868 & 0.9171 \\
DP-DMD   & 0.8372 & 0.8942 & 0.5621 & 0.9562 & 0.9823 & 0.9545 & 0.9655 & 0.8789 \\
Ours    & \textbf{0.9417} & \textbf{0.9538} & 0.6047 & \textbf{0.9814} & \textbf{0.9918} & \textbf{0.9825} & \textbf{0.9893} & \textbf{0.9207} \\
\bottomrule
\end{tabular}%
}
\label{tab:vbench-comparison-vipe}
\end{table*}

\textbf{We refer the readers to the supplementary material for full video comparisons}
\subsection{Additional Results for ablation study}
\begin{table}[h]
\centering
\caption{{Ablation on the effects of GAN loss in our DFD.} Bold indicates the better value. Removing the GAN loss yields results comparable to the model distilled with the GAN loss.
}
\setlength{\tabcolsep}{4pt}
\renewcommand{\arraystretch}{1.2}
\resizebox{0.85\textwidth}{!}{%
\begin{tabular}{lccccccc}
\toprule
\textbf{Model (DFD)} &
\shortstack{\textbf{Subject}\\\textbf{Consistency}} &
\shortstack{\textbf{Background}\\\textbf{Consistency}} &
\shortstack{\textbf{Temporal}\\\textbf{Flickering}} &
\shortstack{\textbf{Motion}\\\textbf{Smoothness}} &
\shortstack{\textbf{Dynamic}\\\textbf{Degree}} &
\shortstack{\textbf{Aesthetic}\\\textbf{Quality}} &
\shortstack{\textbf{Imaging}\\\textbf{Quality}} \\
\midrule
w/o GAN  & \textbf{0.9690 }         & 0.9620 & 0.9785 & 0.9899 & \textbf{0.5000   }       & 0.7194          & \textbf{0.7452} \\
w GAN   &  0.9666 & \textbf{0.9625 }         & \textbf{0.9831 }         & \textbf{0.9912    }    & 0.3750       & \textbf{ 0.7213  } &0.7210      \\
\bottomrule
\label{tab: ablate gan}
\end{tabular}%
}
\label{tab:remove_gan_comparison}
\end{table}

\begin{table}[h]
\centering
\caption{{Ablation on the weights in Eq.~\ref{eq: practical update gradient}.} The results show no clear difference between the two choices of $w$.}
\setlength{\tabcolsep}{4pt}
\renewcommand{\arraystretch}{1.2}
\resizebox{0.85\textwidth}{!}{%
\begin{tabular}{lccccccc}
\toprule
\textbf{Model} &
\shortstack{\textbf{Subject}\\\textbf{Consistency}} &
\shortstack{\textbf{Background}\\\textbf{Consistency}} &
\shortstack{\textbf{Temporal}\\\textbf{Flickering}} &
\shortstack{\textbf{Motion}\\\textbf{Smoothness}} &
\shortstack{\textbf{Dynamic}\\\textbf{Degree}} &
\shortstack{\textbf{Aesthetic}\\\textbf{Quality}} &
\shortstack{\textbf{Imaging}\\\textbf{Quality}} \\
\midrule
$w=\frac{1}{2}$  & \textbf{0.9691   }   & \textbf{0.9661} & 0.9830 & 0.9907 & \textbf{0.5625}        & 0.7135        & \textbf{0.7457 }\\
 $w=1$   & 0.9666 & 0.9625         & \textbf{0.9831}         & \textbf{0.9912}    & 0.3750       & \textbf{0.7213  }  & 0.7210      \\
\bottomrule
\end{tabular}%
}
\label{tab:hybrid_comparison}
\end{table}

\clearpage
\subsubsection{Text Prompts}
\label{app sub sub: text prompts for t2v}
We show part of the text prompts here:

\begin{tcolorbox}[
    colback=gray!5!white,   
    colframe=gray!75!black, 
    title=Example Prompts,  
    arc=3mm,                
    breakable               
]

\begin{itemize}
    \item Pikachu and Eevee dancing joyfully together on a dirt path in a bright, sunlit forest.
    
    \item Cheerful yellow cartoon dragon with green spikes and colorful wings, standing on a bright grassy hill.
    
    \item SpongeBob SquarePants dancing energetically inside his pineapple house, surrounded by floating bubbles.
    
    \item Cartoon-style family of polar bears playing on a floating ice floe.
    
    \item Animated squirrel gathering acorns in a colorful autumn forest.
    
    \item Anime-style mischievous cat batting at a ball of yarn in a cozy living room.
    
   \item Animation of a graceful fox running through a snowy landscape.
    \item Cartoon-style happy dog playing fetch on a bright sandy beach.
\item Animated traditional Japanese village at sunset with glowing paper lanterns and falling cherry blossoms.
\item Animated misty village with green-roofed wooden houses, a cobblestone path, and vibrant orange flowers.
\item An animated princess and men in tuxedos interacting cheerfully by a grand clock tower at night.
\item Cartoon-style elaborate underwater palace made of shells and coral.
\item Animated complex European castle perched on a rugged, mist-shrouded mountain peak.
\item Anime-style bustling Asian street market with colorful awnings, glowing signs, and dense pedestrian traffic.
\item Cartoon-style quirky, multi-level treehouse with suspension bridges and tire swings in a forest.
\item Animation of a historic lighthouse on a rocky coast battling powerful stormy waves.
\item Anime-style Sailor Moon walking down a neon city street, looking back over her shoulder with a smile.
\item Cheerful snowman standing on a snowy roof, looking down at a lit Christmas tree in a town square.
\item Anime close-up of an angry girl with pink twin-tails, golden horns, and a dark purple outfit.
\item Anime scene of a grinning Naruto holding ramen, standing next to a smiling girl with lavender hair.
\item Cartoon blonde mermaid with a white tiara, swimming happily through a vibrant coral reef.
\item Three animated women excitedly examining a map at a table in an ornate room.
\item A cartoon woman smiling and sipping coffee from a red cup at a sunny, cozy café.
\item A romantic animated couple gazing at each other under a Van Gogh-style swirling starry night sky.
\item Two animated princesses having a magical outdoor tea party among lush flowers and glowing sparkles.
\item Two animated princesses dancing joyfully on a palace terrace in the rain under a vibrant rainbow.
\item An animated princess pointing enthusiastically toward a seated queen at a magical royal banquet.
\item An animated princess pointing enthusiastically toward a seated queen at a magical royal banquet.
\item Animation of a street artist creating a detailed vibrant mural on a brick wall.
\item An animated woman in a blue gown gazing from a rose-adorned balcony overlooking a fairy-tale city.
\item A cheerful cartoon princess gesturing joyfully in a vibrant, sunny garden with a distant palace.
\item A tearful cartoon princess in a starry gown looking up at the night sky, with a majestic castle behind her.
\item Cartoon-style towering waterfall plunging into a clear turquoise pool surrounded by jungle foliage.
\item Animation of a vast desert with towering sand dunes stretching to the horizon under a blazing sun.
\item Anime-style dense, ancient forest where beams of sunlight filter dramatically through the high canopy.
\item Cartoon-style bustling cityscape viewed from a high rooftop, with lights blinking on as twilight descends.
\item Animation of a tranquil lake perfectly reflecting the purple and orange colors of a vibrant dawn.
\item Cartoon pig chef wearing a tall white hat, proudly holding a strawberry cake in a modern kitchen.
\item Cartoon-style towering stack of pancakes dripping with maple syrup and melting butter in a cozy diner.
\item Animation of a detailed chef’s market stall with rows of colorful fresh fruits and vegetables.
\item Anime-style detailed close-up of a steaming hot bento box full of intricate, colorful food.
\item \item Cartoon-style conveyor belt with tiny, perfectly formed sushi plates moving past happy diners.
\item Animated bakery window display filled with intricate pastries, artisanal bread, and glowing warmth.
\item Cartoon-style colorful ice cream truck with children lining up on a hot summer day.
\item Anime-style traditional tea ceremony performed in a peaceful garden with gentle steam rising from the cups.
\item Cartoon-style bright red vintage steam train chugging across rolling green hills under a bright sky.
\item Animated futuristic spaceship smoothly landing on a alien planet with purple vegetation.
\item Whimsical stop-motion style yellow submarine exploring deep, glowing underwater ruins.
\item Anime depiction of a sleek sports car speeding through a tunnel illuminated by passing neon lights.
\item Cartoon-style biplane flying dynamic loop-the-loops amidst fluffy, stylized white clouds.
\item Animated pirate ship with detailed sails navigating a vast, sun-drenched ocean at sunset.
\item Cartoon delivery scooter expertly navigating a bustling, rainy city street.
\item Anime-style flying motorcycle soaring high above a dense, green forest canopy.
\item Cartoon-style dusty antique clock shop with dozens of clocks ticking and chiming simultaneously.
\item Animated messy artist’s desk covered in used paintbrushes, tubes of paint, and open sketchbooks.
\item Whimsical stop-motion style pair of worn leather boots resting by a crackling fireplace.
\item \item Anime-style close-up of a well-traveled backpack adorned with colorful keychains and buttons.
\item Cartoon-style bookshelf packed with a vast, eclectic collection of books in a library corner.
\item Animation of a grand grandfather clock with its large pendulum slowly swinging back and forth.
\item Cartoon-style gardening tools—a trowel, gloves, and seed packets—resting on a rustic wooden bench.
\item Anime-style glowing crystal globe positioned on a wooden desk within a wizard’s library.
\item Animated messy child's playroom filled with various toys and building blocks scattered on the floor.
\item Anime-style depiction of a massive magical library where books fly autonomously between high shelves.
\item Animation of a whimsical city made of clouds floating in the sky, connected by glowing bridges of light.
\item Cartoon-style encounter with a friendly forest spirit in a glowing, ancient grove of mossy trees.
\item Anime-style powerful wizard casting a complex spell that generates swirling vortexes of magical energy.
\item Animation of a mischievous imp hiding behind colorful, glowing mushrooms within a deep cave.
\item Cartoon-style friendly dragon and a knight playing chess peacefully under the shade of a large tree.
\item Anime-style phoenix rising dramatically from a bed of glowing ashes with brilliant, fiery feathers.
\item Animation of an intricate city built entirely from colorful crystals that hum and pulse with light and energy.
\item Whimsical cartoon depiction of a celestial garden where stars grow like flowers on glowing vines.
\item Whimsical cartoon depiction of a celestial garden where stars grow like flowers on glowing vines.
\item A young woman with curly red hair laughs while spinning a yellow umbrella in slow motion as rain falls around her on a cobblestone street.
\item Close-up of an elderly man's weathered hands carefully tying a fly fishing lure beside a misty mountain stream at dawn.
\item Cinematic shot of a lone figure in a long coat walking across an endless salt flat under a vast purple twilight sky.
\item A barista pulls a perfect espresso shot, dark crema swirling into a white ceramic cup in soft morning light.
\item Drone footage soaring above a turquoise coastline where waves crash against jagged white cliffs.
\item A chef in a black apron flips a sizzling steak in a cast iron pan, flames leaping briefly upward.
\item Close-up of a honeybee landing on a sunflower, pollen visibly clinging to its legs as it crawls across the petals.
\item A vintage motorcycle rolls slowly down an empty desert highway at sunset, heat shimmering off the asphalt.
\item Cinematic tracking shot of a ballerina in a white tutu rehearsing alone in a sunlit wooden studio.
\item Time-lapse of storm clouds rolling rapidly across a wheat field as the wind bends the golden stalks.
\item A black cat slinks silently across a wet rooftop at night, eyes glowing under a streetlamp.
\item Macro shot of a single dewdrop sliding down a green blade of grass at sunrise.
\item A woman in a flowing red dress runs barefoot along a deserted beach, footprints quickly washed away by the surf.
\item First-person view from inside a glass elevator rising rapidly up a futuristic skyscraper at night.
\item A young boy lights sparklers in his backyard as twilight settles, his face illuminated in flickering gold.
\item Aerial shot of a small wooden boat carving a wake across a glassy lake surrounded by dense pine forest.
\item A potter's hands shape wet clay on a spinning wheel, water and clay slipping between their fingers.
\item Cinematic close-up of an astronaut's helmet visor reflecting the curve of Earth from low orbit.
\item A street musician plays a worn acoustic guitar on a busy sidewalk as commuters hurry past in motion blur.
\item Slow-motion shot of a basketball swishing through a chain net on an outdoor court at golden hour.
\item A fox steps cautiously into a snowy clearing at dusk, ears twitching as snowflakes drift down around it.
\item Wide shot of a Tibetan monastery perched on a cliff edge, prayer flags whipping in the high mountain wind.
\item A barista pours latte art into a cappuccino, viewed from directly above as a delicate rosetta forms.
\item A child runs through a sprinkler in a suburban backyard on a hot summer afternoon, laughing.
\item Macro footage of ink spreading and blooming through clear water, forming twisting black tendrils.
\item Cinematic shot of a saxophonist playing on a foggy New Orleans street corner, neon signs glowing behind him.
\item An elderly couple slow dances in their kitchen as warm afternoon light streams through gauzy curtains.
\item Drone shot pulling back from a single red farmhouse to reveal endless rolling vineyards stretching to the horizon.
\item A blacksmith hammers a glowing orange blade on an anvil, sparks flying with each strike.
\item Underwater shot of a sea turtle gliding gracefully through shafts of sunlight piercing a coral reef.
\item A businesswoman in a tailored suit strides confidently across a rain-soaked Tokyo crosswalk at night, neon reflecting in puddles.
\item Close-up of a vinyl record dropping onto a turntable, the needle settling as it begins to spin.
\item Time-lapse of a city skyline transitioning from sunset to night, windows lighting up across the buildings.
\item A surfer paddles into a massive wave at dawn, silhouetted against a pink and orange sky.
\item Cinematic shot of a horse galloping across a misty Scottish moor, hooves throwing up clumps of damp earth.
\item A jazz pianist's hands move expressively across the keys in a dimly lit speakeasy filled with soft amber light.
\item A baker dusts flour across a freshly braided loaf of bread on a wooden countertop in a sunlit kitchen.
\item Slow-motion shot of a tennis ball striking a racket, the strings flexing dramatically on impact.
\item A father teaches his young daughter to ride a bicycle along a tree-lined neighborhood street in autumn.
\item Aerial view of a fishing village at dawn, colorful boats returning to harbor across calm waters.
\item Macro shot of a butterfly slowly opening and closing its wings on a lavender stem in a summer breeze.
\item A man reads a leather-bound book in a worn armchair beside a crackling fireplace as snow falls past the window.
\item Cinematic shot of a train emerging from a tunnel into a sunlit mountain valley filled with wildflowers.
\item A surgeon in scrubs washes their hands at a hospital sink, water cascading over their fingers in clinical white light.
\item Close-up of raindrops striking a car windshield, the wipers sweeping rhythmically across the glass.
\item A florist arranges a bouquet of peonies and eucalyptus on a marble counter in a bright minimalist shop.
\item A lone hiker reaches a mountain summit at sunrise and stands silhouetted against the rising sun, arms outstretched.
\item Slow-motion shot of a glass shattering as it strikes a tile floor, shards spreading outward.
\item Drone footage following a red kayak winding through a narrow emerald-green river canyon.
\item A young man in a leather jacket leans against a vintage car at a 1950s-style diner, neon signs buzzing overhead.
\item Cinematic shot of two dancers performing a tango in an empty warehouse lit by a single overhead bulb.
\item A child's hand reaches up to touch the glass of an aquarium as a stingray glides past on the other side.
\item Wide shot of a hot air balloon festival at dawn, dozens of colorful balloons inflating across a misty field.
\item A glassblower shapes molten glass at the end of a long iron rod, the orange glow lighting their focused face.
\item Close-up of espresso dripping into a clear glass mug, steam curling upward in soft window light.
\item A street artist spray-paints a vibrant mural on a brick wall in an alley, bystanders watching in the background.
\item Slow-motion shot of a diver leaping from a high cliff into a deep blue ocean below.
\item A woman pours steaming tea from a cast iron kettle into a delicate porcelain cup at a wooden table.
\item Aerial shot of an autumn forest in full color, a winding road cutting through a fiery red and orange canopy.
\item A construction worker welds a steel beam high above a city skyline, sparks raining down against the sunset.
\item Cinematic shot of a vintage typewriter on a desk by a rainy window, paper slowly being pulled from the carriage.
\item A skateboarder grinds along a metal handrail in an urban plaza, captured in slow motion at golden hour.
\item A grandmother teaches her granddaughter to roll dumplings at a flour-dusted kitchen counter, both smiling.
\item Underwater shot of a school of silver fish parting around a diver swimming through a sunlit kelp forest.
\item Time-lapse of a flower blooming from a tight bud to a full open blossom over the course of a day.
\item A scientist examines a glowing blue liquid in a test tube inside a softly lit laboratory.
\item Cinematic wide shot of a desert caravan crossing massive sand dunes at sunset, long shadows trailing behind.
 \item A coffee shop window steams up from the inside as rain streaks down the outside, a person reading at the counter.
\item Slow-motion shot of a flock of birds taking off from a marsh at dawn, wings beating against pink mist.
\item A young woman pins a Polaroid photo onto a string of lights in her bedroom, the photo gently swinging.
 \item Drone footage gliding low over a frozen lake at dusk, a single ice skater carving graceful arcs across its surface.
\end{itemize}

\end{tcolorbox}

\subsection{Societal Impacts}
Our work on distilling video diffusion models presents both positive and negative potential societal impacts. On the positive side, our method significantly reduces the computational cost and inference time required for high-fidelity video generation. This democratizes access to creative tools, allowing users with limited compute resources to generate videos, while also reducing the carbon footprint and energy consumption associated with deploying large-scale generative models. On the negative side, accelerating video generation inherently scales the risks associated with the base models. Faster generation can facilitate the rapid creation of deepfakes, misinformation, and malicious content. While our theoretical framework improves algorithmic efficiency, it does not inherently prevent these misuses. Safe deployment of these distilled models in real-world applications will require coupling them with robust safety filters, provenance tracking, and watermarking techniques.

\end{document}